\documentclass[journal]{IEEEtran}
\usepackage{graphicx}
\usepackage{amsfonts,amssymb}
\usepackage{lipsum}
\usepackage{color}
\usepackage{bm}
\usepackage{tabu}
\usepackage{subeqnarray}
\usepackage{algorithm}
\usepackage{amsthm}
\usepackage{algorithmic}
\usepackage{caption}
\usepackage{mathrsfs}
\usepackage{cite}
\usepackage{epstopdf}
\usepackage{amsmath}
\graphicspath{{figures/}}
\usepackage{booktabs}
\usepackage{diagbox}
\makeatletter

\newcommand{\Rmnum}[1]{\expandafter\@slowromancap\romannumeral #1@}
\makeatother

\begin{document}

\title{Context-Aware Online Learning for Course Recommendation of MOOC Big Data }
\author{\normalsize Yifan Hou, \emph{Member, IEEE}, \normalsize Pan Zhou, \emph{Member, IEEE}, Ting Wang,   Li Yu, \emph{Member, IEEE},   Yuchong Hu, \emph{Member, IEEE}, Dapeng Wu,
\emph{Fellow, IEEE}

\thanks{Yifan Hou, Pan Zhou and Li Yu are with School of Electronic Information and Communications,
Huazhong University of Science and Technology, Wuhan 430074, China.

Ting Wang is with Computer Science and Engineering, Lehigh University,   PA 18015, USA.

Yuchong Hu is with School of Computer Science and Technology,
Huazhong University of Science and Technology, Wuhan, 430074 China.

Dapeng Oliver Wu is with Department of Electrical and Computer Engineering,
University of Florida, Gainesville, FL 32611, USA.

Contacting email: panzhou@hust.edu.cn

This work was supported by the National Science
Foundation of China under Grant 61231010, 61401169, 61529101 and CNS-1116970.
}
}

\maketitle
\thispagestyle{empty}
\pagestyle{plain}

\begin{abstract}
The Massive Open Online Course (MOOC) has expanded significantly in recent years. With the widespread of MOOC, the opportunity to study the fascinating courses for free has attracted numerous people of diverse educational backgrounds all over the world.
In the big data era, a key research topic for MOOC
is how to mine the needed courses in the massive course databases in cloud for each
individual student accurately and rapidly as the number of courses is increasing fleetly. In this respect, the key challenge is
how to    realize personalized course recommendation as well as to reduce the
computing and storage costs  for the tremendous course data.
In this paper, we propose a big data-supported, context-aware online learning-based
 course recommender system that could handle the dynamic and \emph{infinitely} massive
  datasets, which recommends courses by using personalized context information and historical statistics.
The context-awareness takes the personal preferences into consideration, making the recommendation suitable for  people with different
backgrounds.
Besides, the algorithm achieves the sublinear regret performance, which means it can gradually
recommend the mostly preferred and matched  courses to students.
In addition, our  storage module is expanded to the distributed-connected storage nodes, where the
devised algorithm can handle massive course storage problems from
heterogeneous sources of course datasets. Comparing to existing  algorithms, our proposed algorithms achieve  the \emph{linear} time complexity and space complexity.
 Experiment results verify the superiority of our algorithms when comparing with existing ones in the MOOC big data setting.

\end{abstract}


%
\IEEEpeerreviewmaketitle


\section{Introduction}
MOOC is a concept first proposed in 2008 and known to the world in 2012\cite{In0.1}\cite{In0.2}.
Not being accustomed to the traditional teaching model or being desirous to find a unique learning style, a growing number of people have partiality for learning on MOOCs.
Advanced thoughts and novel ideas give great vitality to MOOC, and over 15 million users have marked in Coursera\cite{Coursera} which is a platform of it.
Course recommender system helps  students to find the requisite courses directly in the course ocean of numerous MOOC platforms such like Coursera, edX, Udacity and so on\cite{In1}.
However,  due to the rapid growth rate of users, the amount of needed courses has been expanding continuously.
And according to the survey about the completion rate of MOOC \cite{4rate}, only 4\% people finish their chosen courses.
Therefore, finding a preferable course resource  and locating it in the massive data bank, e.g., cloud computing and storage platforms, would be a
daunting
``needle-in-a-haystack" problem.


One key challenge in future MOOC course recommendation is processing tremendous data that
 bears the feature of  volume, variety, velocity, variability and veracity\cite{big_data1} of big data. Precisely,
the recommender system for MOOC big data needs to handle the dynamic changing and nearly infinite course data with heterogeneous sources and prior unknown scale   effectively. Moreover, since the Internet and cloud computing services are turning in the direction of
supporting different users around the world, recommender systems are necessary  to consider the features of students, i.e. cultural difference, geographic disparity and education level,  , one has his/her unique preference in
evaluating a course in MOOC. For example, someone pays more attention to the quality of exercises while the other one   focuses on the classroom rhythm more.
We use the concept of $context$ to represent those mentioned features as the students' personalized information.
The context space is encoded as a multidimensional space ($d_X$ dimensions),   where $d_X$ is the number of features.
As such, the recommendation becomes student-specific, which could improve  the recommendation accuracy.
Hence, appending context information to the models  for processing the online courses is  ineluctable\cite{In2}\cite{In3}.

Previous context-aware algorithms such as \cite{ACR} only perform well with the known scale of recommendation datasets.
Specifically, the algorithm in \cite{ACR} would rank  all courses   in MOOC  as leaf nodes, then it clusters some relevance courses together as their parent nodes based on the historical information and current users' features.
The algorithm keeps clustering the course nodes and building their parent nodes  until the root node (bottom-up design).   If there comes a new course, all the nodes are changed and needed to compute again. As for the MOOC big data, since the number of courses keeps increasing and becoming fairly tremendous,    algorithms in \cite{ACR} are prohibitive to be applied.

Our main theme in this paper is   recommending courses in tremendous datasets to students in real-time based on their preferences.
The course data are stored in course cloud and new courses can be loaded at any time.
We devise  a top-down  binary tree to denote and record the process of partitioning course datasets, and every node in the tree is a set of courses.
Specifically, there is only one root course node including all the courses in the binary tree at first.
   The  course scores feedback from students in  marking system are denoted as rewards.
Every time a course is recommended, a reward which is used to improve the next recommending accuracy is fed back from the student.
The reward structure consists as a unknown stochastic function of context features and course features at each
recommendation, and our algorithm concerns the expected reward of
every node in the long run.
Then the course binary tree divides the current node into two child nodes and selects one course randomly in the node with the current best expected value.
It omits most of  courses in the node that would not be selected to greatly improve the learning performance.
 It also supports incoming new courses  to the existing nodes as unselected items without changing the current built tree pattern.


However, other challenges influencing  on the recommending accuracy still remain.
In practice, we observe that the number of courses keeps increasing and the in-memory storage cost of one online
course is about $1GB$ in average, which
 is fairly large.
Therefore, how to store the tremendous course data and how to process the course data effectively become a challenge.
Most previous works \cite{Rb43}\cite{ACR} could only realize the \emph{linear} space complexity, however it's not promising for MOOC big data.
We propose a distributed storage scheme to store the course data with many distributed-connected storage units in the course cloud.
For example, the storage units may be divided based on the platforms of MOOC.
On the one hand, this method can make invoking process effectively with little extra costs on course recommendation.
On the other hand, we prove the space complexity can be bounded \emph{sublinearly} under the optimal condition (the number of units satisfies certain relations) which is much better than\cite{ACR}.

In summary, we propose an effective context-aware  online
learning algorithm   for  course big data recommendation to offer  courses to students in MOOCs. 
The main contributions  are listed as follows:

\begin{itemize}

	\item  The algorithm can accommodate to highly-dynamic increasing course database environments, realizing the real big data support by the course tree that could index nearly infinite and dynamic changing datasets.

	\item  We consider context-awareness for personalized course recommendations, and devise an effective
context partition scheme that greatly improves the learning rate and recommendation accuracy for different   featured students.
	
	\item  Our proposed distributed storage model stores data with distributed units rather than single storage carrier, allowing the system to utilize the course data better and performing well with huge amount of data.

	\item  Our algorithms
enjoy superior time and space complexity.  The time complexity is bounded linearly, which means they achieve a higher learning rate than previous  methods in MOOC.
	For the space complexity, we  prove that it is linear in the primary algorithm and could be sublinear in distributed storage algorithm under the optimal condition.

\end{itemize}


The reminder of the paper is organized as follows.
Section \Rmnum{2} reviews related works and compares with our algorithms.
Section \Rmnum{3} formulates the recommendation problem and algorithm models.
Section \Rmnum{4} and Section \Rmnum{5} illustrate our algorithms and bound their regret.
Section \Rmnum{6} analyzes the space complexity of our algorithms and compares the theoretical results with existing works.
In Section \Rmnum{7}, we verify the algorithms by experiment results and compare with relevant previous algorithms \cite{ACR}\cite{HCT}.
Section \Rmnum{8}  concludes the paper.


\section{Related Works}


A plethora of previous works exist on recommending algorithms.
As for MOOC, two major tactics to actualize the algorithms are filtering-based approaches and online learning methods\cite{R1.1}.
 Apropos of filtering-based approaches, there are some branches such like collaborative filtering\cite{In6}\cite{R1.2}, content-based filtering\cite{R1.4} and hybrid approaches\cite{R1.5}\cite{R1.6}.
%
%
The collaborative filtering approach gathers the students' learning records together and then classifies them into groups based on the characteristics provided, recommending a course from the group's learning records to new students\cite{R1.2}\cite{In6}.
Content-based filtering recommends a course to the student which is relevant to the learning records before\cite{R1.4}.
Hybrid approach is the combination of the two methods.
The filtering-based approaches can perform better at the beginning than online learning algorithms.
However, when the data come to very large-scale or become stochastic, the filtering-based approaches lose the accuracy and become incapable of utilizing the history records adequately.
Meanwhile, not considering the context makes the method unable to recommend courses precisely by taking every student's preference into account.

Online learning can overcome the deficiencies of filtering-based approaches.
Most previous works of recommending  courses utilize the adaptive learning\cite{R2.1}\cite{R2.2}\cite{R2.3}.
In\cite{R2.1}, the CML model was presented. This model combines the cloud, personalized course map and adaptive MOOC learning system together, which is quite comprehensive for the course recommendation with context-awareness. Nevertheless, as for big data, the model is not efficient enough since these works could not handle dynamic datasets and they may have a prohibitively high time cost   with near ``infinite" massive datasets.
Similar works are widely distributed in \cite{Rb42,Rb43,Rb45,Rb46,Rb47} as contextual bandit problems.
In these works, the systems know the rewards of selected ones and record them every time, which means the course feedback can be gathered from students after they receive the recommended courses.
There is no work before that realizes contextual bandits with infinitely increasing datasets.
Our work is motivated from  \cite{RbX} for big data support bandit theory, but \cite{RbX} is not context-aware. We consider
the context-aware online learning for the first time with delicately devised context partition
schemes for MOOC big data. 

\section{Problem Formulation}
\begin{figure}
	\centering
	\includegraphics[scale=.22]{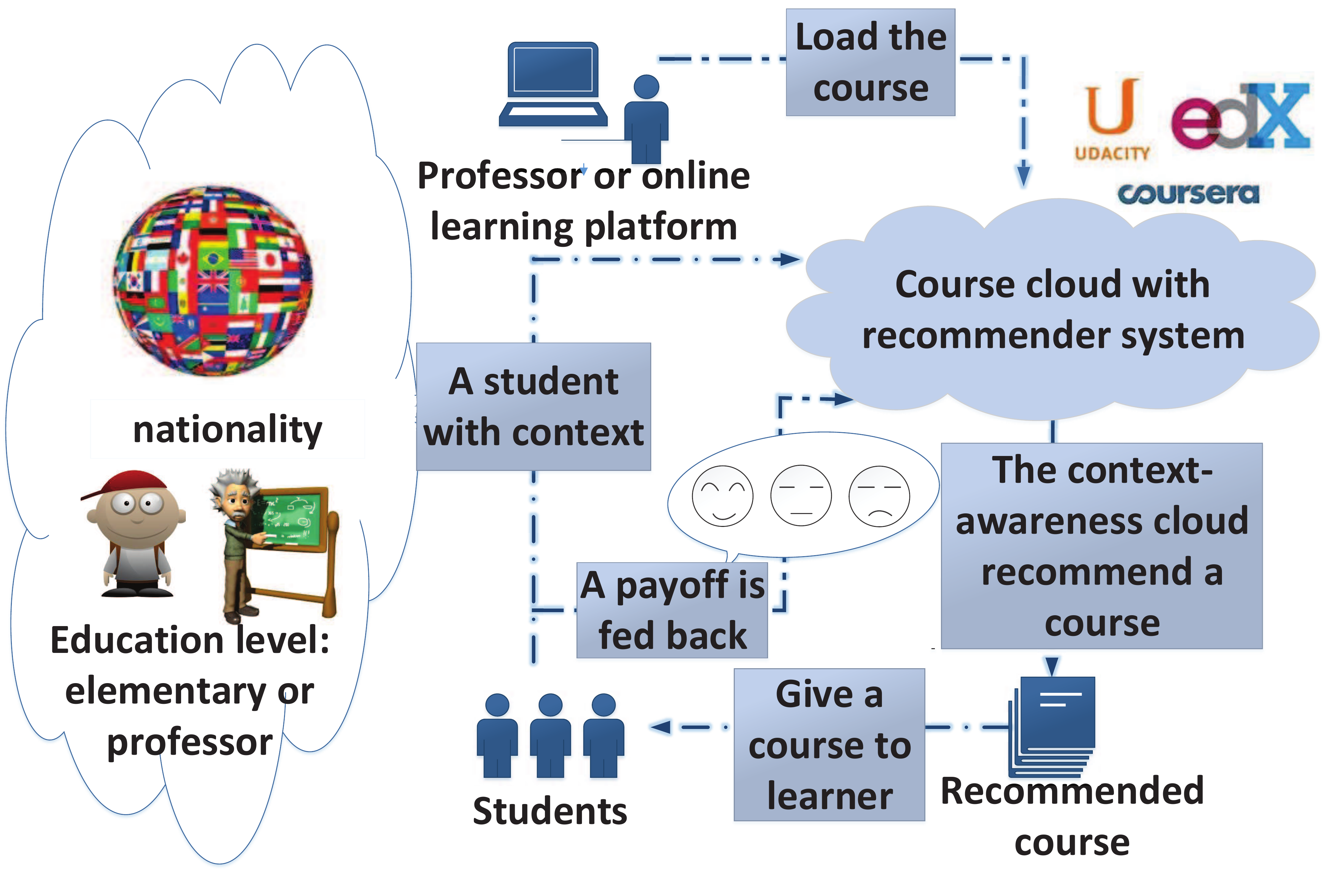}
	\caption{MOOC Course Recommendation and Feedback System}
	\label{fig:digraph1}
\end{figure}
In this section, we present the system model, context model, course model and the regret definition.
Besides, we define some relevant notations and preliminary definitions.



\subsection{System Model}

Fig. 1 illustrates our model of operation.
At first, the professors upload the course resources to the course cloud, where the uploaded courses are indexed by the set $\mathcal{C}=\{ c_1, \, c_2,  ... \}$ whose elements are vectors with dimension $d_C$ representing the number of course features.
As for the users, there are consciously incoming students over time which are denoted as $\mathcal{S} = \{s_1, \, s_2,... \}$.
Then, the system collects context information of students.
We denote the set of context information of students as $\mathcal{X} = \{x_1, \, x_2, \, ...,  x_i, ..., \}$, where $x_i$ is the vector in context
 space $\mathcal{X}$. 

We use time slots $t = 1, \, 2,..., \, T$ to denote rounds.
For simplicity, we use $s_t$, $x_t$, $c_t$ to denote the current incoming student, the student context vector and the recommended course at time $t$.
In each time slot $t$, there are three running states: (1) a student $s_t$ with an exclusive context vector $x_t$ comes into our model; (2) the model recommends a course $c_t$ by randomly selecting one from the current course node to the student $s_t$;
(3) the student $s_t$ provides feedback   due to the newly recommended course $c_t$ to the system.

We assume the context sequence that generating the rewards of courses follows an i.i.d. process, otherwise
if there are mixing  within the sequence in practice, we could use the technique in  \cite{Rdb} by using two i.i.d sequences  to bound the
mixing process without much performance difference.
The  $r_{x_i, c_j}(t)$   denotes the feedback reward from the student with context $x_i$ of course $c_j$ at time $t$.
For the recommending process, first there comes a student $s_k$ with context vector $x_i$.
Then the system recommends a course $c_j$ to the student $s_k$ based on the historical reward information and context vector $x_i$,
after that the student $s_k$ gives a new reward $r_{x_i,c_j}(t)$ to the system.
We define $r_{x_i,c_j}(t)=f(x_i,c_j) + \varepsilon_t $,
where $\varepsilon_t $ is a bounded noise with $\mathbb{E}[\varepsilon_t|(x_i,c_j)]=0$ and $f(x_i,c_j)$ is a function of two variables ($x_i$, $c_j$).
Besides, we normalize the reward as $r_{x_i,c_j}(t) \in [0,1] $.

Fig. 2 illustrates the relationship between context vector $x_i$ and course vector $c_j$ over reward.
To better illustrate the relations, we degenerate the dimensions of them as $d_X  = d_C  = 1$. Practically, we have the reward axis with dimension $1$.
Thus,  we take the context vector and course details as two horizontal axes in a space rectangular coordinate system.
From the schematic diagram in Fig. 2 at time slot $t_0$, the reward varies in the context axis and course axis.
To be more specific, for a determined student $s_{k_0}$ whose context $x_{i_0}$ is unchanged, the reward $r_{x_{i_0},c_{j}}(t_0)$ differs from courses $c_j$ shown in blue plane coordinate system.
On the other hand, for a determined course $c_{j_0}$ shown in crystal plane coordinate system, people with different context $x_{i}$ have different rewards $r_{x_i,c_{j_0}}(t_0)$ of courses.



\begin{figure}
	\centering
	\includegraphics[scale=.2]{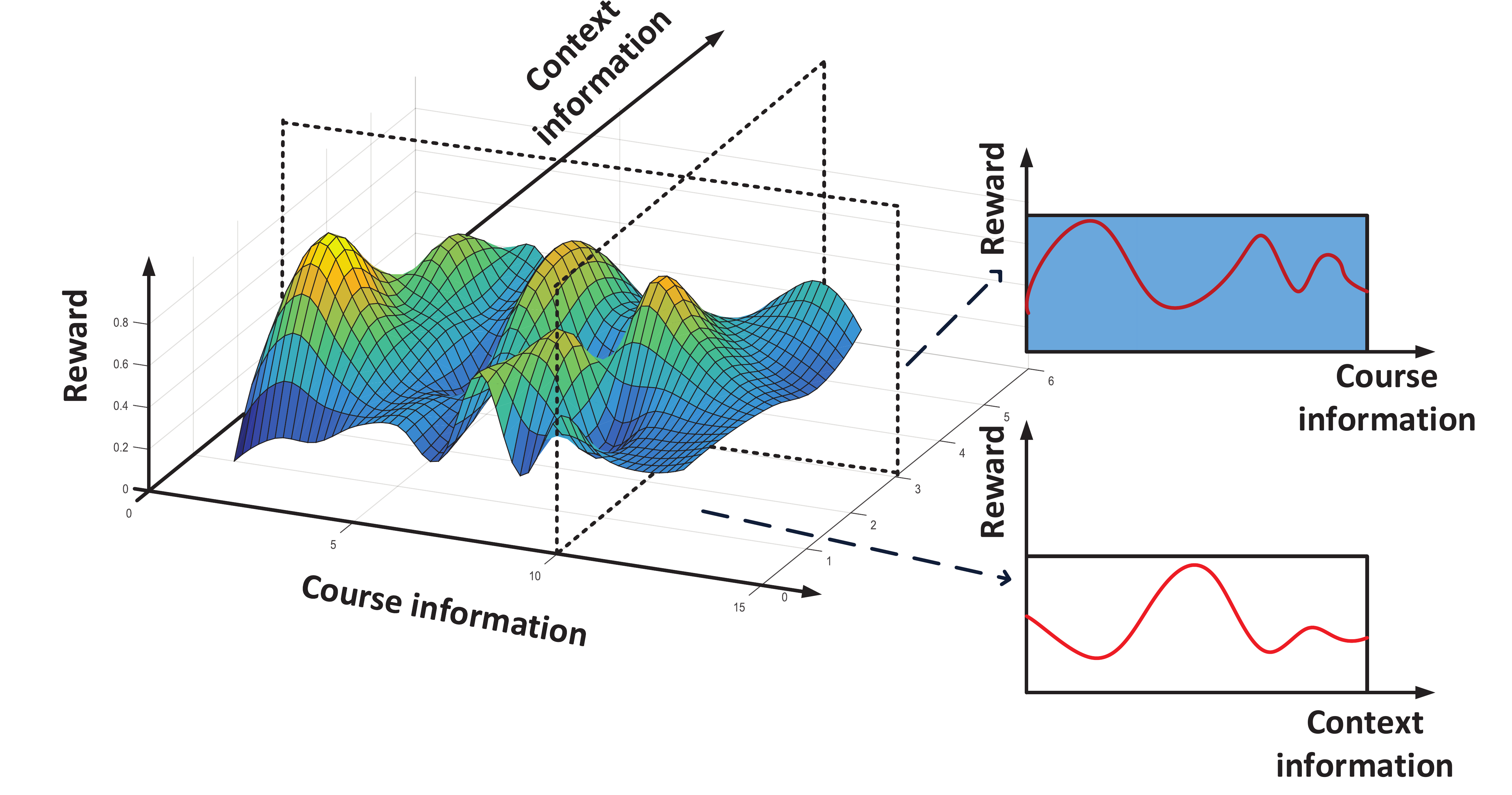}
	\caption{Context and Course Space Relation Schema  at slot $t_0$}
	\label{fig:digraph2}
\end{figure}


\subsection{Context Model for Individualization}

The context space is a $d_X$-dimensional space which means the context $x_i$ $ \in $ $\mathcal{X}$ is a vector with $d_X$ dimensions.
The $d_X$-dimensional vectors encode features such as ages, cultural backgrounds, nationalities, the educational level, etc., representing the characteristics of the student.
We normalize every dimension of context range from 0 to 1, e.g., educational level ranges from $[0,1]$ denoting the educational level from the elementary to  the expert in the related fields.
With the normalization in each dimension, we denote the context space as $\mathcal{X}$ = ${[0,1]^{{d_X}}}$, which is a unit hypercube.
As for the difference between two contexts,
${D_X}$($x_i, x_j$) is used to delegate the dissimilarity between context $x_i$ and $x_j$.
We use the Lipschitz condition to define the dissimilarity.
\newtheorem{myAssumption}{\textbf{Assumption}}
\begin{myAssumption}
	There exists constant ${L_X} > 0$ such that for all context $x_i,x_j \in $ $\mathcal{X}$, we have ${D_X}$($x_i$,$x_j$) $ \le {L_X}{{||x_i  -  x_j|}}{{{|}}^\alpha }$, where $|| \bullet ||$ denotes the Euclidian norm in ${{\cal \mathbb{R}}^{{d_X}}}$.
\end{myAssumption}
Note that the Lipschitz constants $L_X$
are not required to be known by our recommendation algorithms. They will only be used in quantifying the learning algorithms'  performance.
As for the parameter $\alpha$, it's referred to as similarity information\cite{Rb43} and we assume that it's known by the algorithms that
qualify the degree of similarity  among courses.
We present the context dissimilarity mathematically with ${L_X}$ and $\alpha$ and they will appear in our regret bounds.

To illustrate the context information precisely, we define the \emph{slicing number} of context unit hypercube as ${n_T}$, indicating the number of sets in the partition of the context space $\mathcal{X}$.
With the \emph{slicing number} $n_T$, each dimension can be divided into $n_T$ parts, and the context space is divided into $({n_T})^{d_X}$ parts where each part is a $d_X$-dimensional hypercube with dimensions ${1\over{n_T}} \times {1\over{n_T}} \times ... {1\over{n_T}}  $.
To have a better formulation, $\mathcal{P}_T$ = \{${P_1}$,${P_2}$,$...$,${P_{{{({n_T})}^{{d_X}}}}}$\} is used to denote the sliced chronological sub-hypercubes, and we use $P_t$ to denote the sub-hypercube selected at time $t$.
As illustrated in Fig. 3, we let ${d_X} = 3$ and ${n_T} = 2$.
We divide every axis into 2 parts and the number of sub-hypercubes is ${({n_T})^{{d_X}}} = 8$.
For the simplicity, we use the center point $x^{P_t}$ in the sub-hypercube ${P_t}$ to represent the specific contexts $x_t$ at time $t$.
With this model of context, we divide the different users into ${({n_T})^{{d_X}}}$ types.
For simplicity, when $P_t$ is used in the upper right of the notation, it means that the notation is in the sub-hypercube $P_t$ which is selected at time $t$, and the subscript ``$*$" means the optimal solution over that notation.
\subsection{Course Set Model for Recommendation}

We model the set of courses as a ${d_C}$-dimensional space, where $d_C$ is a constant to denote the number of all courses features e.g. language, professional level, provided school in $\mathcal{C}$.
We set every course in $\mathcal{C}$ as a $d_C$ dimensional vector, and for the newly added dimensions of courses, the value is set  as $0$.
Similar to the context, we define the dissimilarity of courses as ${D^{P_t}_C}$($c_i, c_j$) to indicate the farthest relativity between the two courses $c_i$ $c_j$ belonging to any the context vectors $x_i \in P_t$ at time $t$, where the context vector $x_t$ belongs to the context sub-hypercube $P_t$.
\newtheorem{myDefinition}{\textbf{Definition}}
\begin{myDefinition}
	Let ${D^{x_t}_C}$ over $\mathcal{C}$ be a non-negative mapping (${{\cal C}^2}$ $ \to $ ${\cal \mathbb{R}}$):
	${D^{P_t}_C}(c_i, c_j) = \mathop {\sup }\nolimits_{x_i  \in P_t } D_C^{x_i } (c_i ,c_j ) $ ,
	where ${D^{P_t}_C}(c_i,c_j) = {D^{x_t}_C}(c_i,c_j) = 0$ when $i=j$.
\end{myDefinition}
We assume that the two courses which are more relevant have the smaller dissimilarity  between them.
For example, the courses taught both in English have closer dissimilarity than the courses with different languages when concerning the language feature of course.


\begin{figure}
	\centering
 	\includegraphics[scale=0.11]{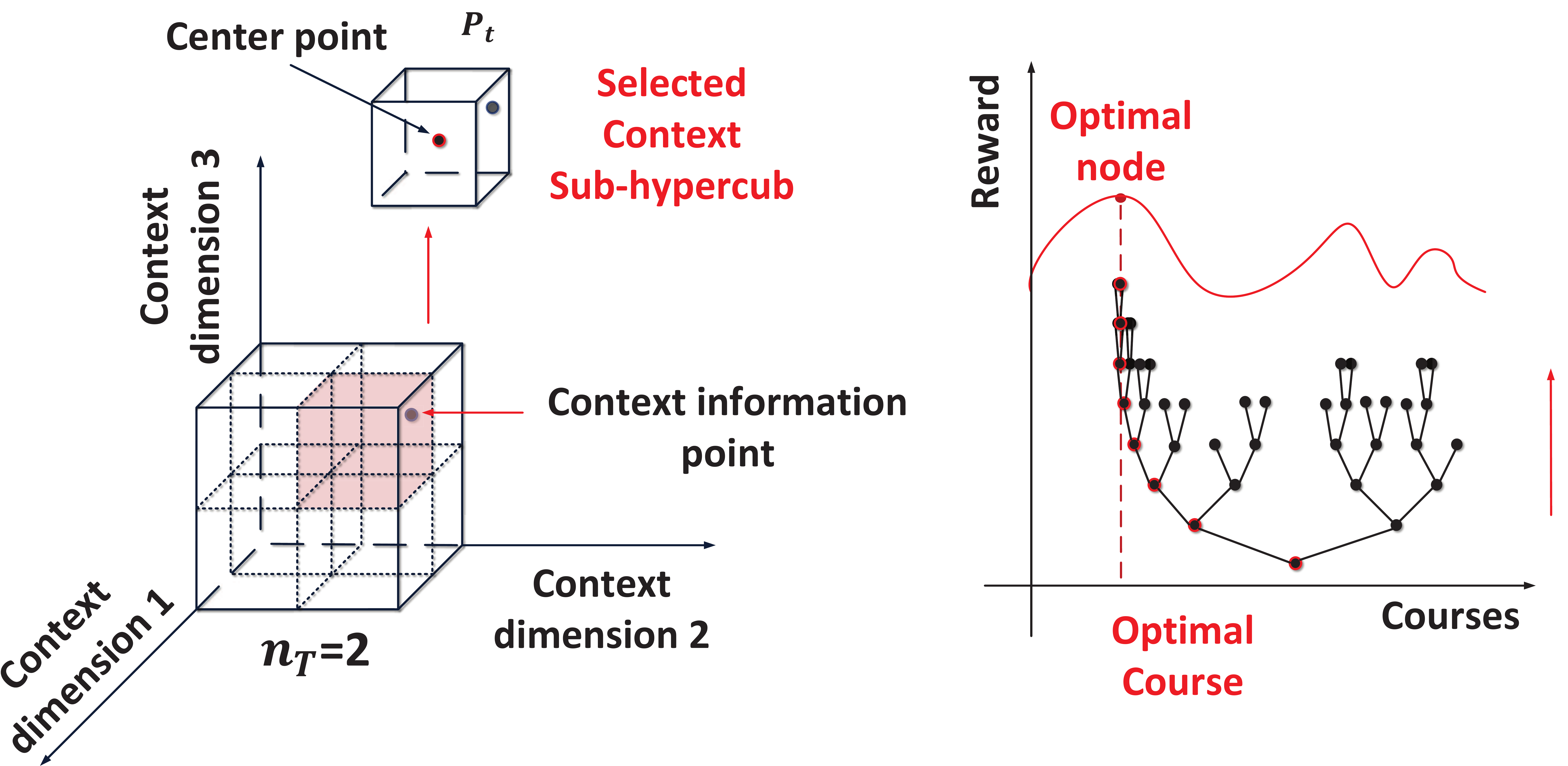}
	\caption{Context and Course Partition Model}
	\label{fig:digraph3}
\end{figure}

As for the course model, we use the binary tree whose nodes are associated with subsets of $\mathcal{X}$ to index the course dataset.
We denote the nodes of courses as $$\{ N_{h,i}^{P_t } |1 \le i \le 2^h ;h = 0,1...; \forall \; P_t  \in P_T \} .$$
Let ${{N}^{P_t}_{h,i}} $ denote the nodes in the depth $h$ and ranked $i$ from left to right in context sub-hypercube $ P_t $ which is selected at time $t$, where the ranked number $i$ of nodes at depth $h$ is restricted by $1 \le i \le {2^h}$.
We let ${\mathcal{N}^{P_t}_{h,i}} \in \mathcal{X} $ represent the course region associated with the node ${N}^{P_t}_{h,i}$.
The region of root node ${N}_{0,1}$ of the binary course tree is a set of the whole courses
${\mathcal{N}^{P_t}_{0,1}} = \mathcal{C}.$
And with the exploration of the tree, the region of two child nodes contains all the courses from their parent region, and they never intersect with each other,
${\mathcal{N}^{P_t}_{h,i}} = {\mathcal{N}^{P_t}_{h+1,2i-1}} \cup {\mathcal{N}^{P_t}_{h+1,2i}},$
${\mathcal{N}^{P_t}_{h,i}} \cap {\mathcal{N}^{P_t}_{h,j}} = \emptyset  $ for any $i \ne j. $
Thus, the $\mathcal{C}$ can be covered by the regions of ${\mathcal{N}^{P_t}_{h,i}}$ at any depth
${\mathcal{C} = }\mathop  \cup \nolimits_1^{{{2}^{h}}} {\mathcal{N}^{P_t}_{h,i}}$.
To better describe the regions, we define the $diam({\mathcal{N}^{P_t}_{h,i}}$) to indicate the size of course regions,
$$diam({\mathcal{N}^{P_t}_{h,i}}) \!=\! \mathop {\sup}\nolimits_{c_i,c_j \in {\mathcal{N}^{P_t}_{h,i}}} \! {D^{P_t}_C}(c_i,c_j){{ }} \; f\!or\;any\; c_i,c_j \! \in \! {\mathcal{N}^{P_t}_{h,i}}.$$
The dissimilarity ${D^{P_t}_C}$($c_i, c_j$)  between courses $c_i$ and $c_j$ can be represented as the gap between course languages, course time length, course types and any others which indicate the discrepancy.
We denote the size of regions $diam({\mathcal{N}^{P_t}_{h,i}})$ with the largest dissimilarity in the course dataset $ {\mathcal{N}^{P_t}_{h,i}}$ for any context $x_i \in {P_t}$.
Note that the $diam$ is based on the dissimilarity, and that can be adjusted by selecting different mappings.
For our analysis, we make some reasonable assumptions as follows.
We define the set $\mathcal{M}=\{ m^{P_1}, \, m^{P_2},...m^{ {({n}_T)}^{d_X} }  \}$ as the parameter to bound the size of regions of nodes in context sub-hypercube ${P_t}$, where all the elements in $\mathcal{M}$ satisfy $m^{P_t} \in (0,1)$.
For simplicity, we take $m$ as the maximum in $\mathcal{M}$, which means $m = \max \{ m^{P_t} | m^{P_t} \in \mathcal{M}\} $.
\begin{myAssumption}
	For any region ${\mathcal{N}^{P_t}_{h,i}}$, there exists constant $\theta \ge 1$, $k_1$ and  $m  $, where we can get $ {k_1\over\theta} {(m)^h}\le diam({\mathcal{N}^{P_t}_{h,i}}) \le k_1{(m)^h}.$
\end{myAssumption}

With Assumption 2 we can bound the size of regions with $k_1{(m)^h},$ which accounts for the maximum possible variation of the reward over ${\mathcal{N}^{P_t}_{h,i}}$.
Due to the properties of binary tree, the number of regions increases exponentially with the depth rising, where using the exponential decreasing term $k_1{(m)^h}$ to bound the size of regions is reasonable.
We use the mean reward $f(x_i, c_j)$ to handle the model.
Based on the concept of the region and reward, we denote the courses in ${\mathcal{N}^{P_t}_{h,i}}  $ as $c^{P_t}_t(h,i)$ at time $t$ in the context sub-hypercube $P_t$.
Since there are tremendous courses and it is nearly impossible to find two courses with equal reward,
for each context sub-hypercube ${P_t}$, there is only one overall optimal course defined as ${P_t}$ as $c^{P_t *}  = \mathop {\arg \max }\nolimits_{c_j  \in \mathcal{C} } f(r_{c_j }^{P_t }) $ and each region ${\mathcal{N}^{P_t}_{h,i}}$ has a local optimal course defined as ${P_t}$ as $c^{P_t *}(h,i)  = \mathop {\arg \max }\nolimits_{c_j  \in {\mathcal{N}_{h,i}^{P_t} } } f(r_{c_j }^{P_t }) ,$
where we let $f(\bullet)$ be the mean value, i.e., $f(r_{x_i,c_j}) = \mathbb{E}[f(x_i,c_j)+\varepsilon_t] = f(x_i,c_j)$ and $r_{c_j}^{P_t}$ means the $r_{x_t,c_j}$ in ${P_t}$.


\subsection{The Regret of Learning Algorithm}
Simply, the regret $\mathbb{R}(T)$ indicates the loss of reward in the recommending procedure due to the unknown dynamics.
As for our tree model, the regret $\mathbb{R}(T)$ is based on the regions of the selected tree nodes ${\mathcal{N}^{P_t}_{h,i}}$.
In other words,  the regret $\mathbb{R}(T)$ is calculated by the accumulated reward difference  between recommended courses $c_t$ and the optimal course $c^{P_t*}$ with context $x_t$ over reward in the context sub-hypercube $P_t$ at time $t$,
thus we define the regret as
\begin{IEEEeqnarray*}{l}
	\!\!\!\!\!\begin{array}{l}
		\mathbb{R}(T) = \sum\limits_{t = 1}^T f({r^{P_t}_{c^{P_t*}}})  - \mathbb{E}\left[\sum\limits_{t = 1}^T  r^{P_t}_{x_t,c_t}(t) \right],
	   \IEEEyesnumber \label{eq:regret}
	\end{array}
\end{IEEEeqnarray*}
where $r_{c^{P_t*}}^{P_t}$ is the reward of optimal course in ${P_t}$ and $r^{P_t}_{x_t,c_t}$ is the reward of course $c_t$ with context $x_t$ in ${P_t}$.
Regret shows the convergence rate of the optimal recommended option.
When the regret is sublinear $   \mathbb{R}(T)  = O {{({T}}^\gamma) }$ where $ 0 < \gamma  < 1 $, the algorithm will finally converge to the best course towards the student.
In the following section we will propose our algorithms with sublinear regret.



\section{Reformational Hierarchical Tree}

In this section we propose our main online learning algorithm to mine courses in MOOC big data.

\subsection{Algorithm of Course Recommendation}


%
\begin{algorithm}
	\caption{Reformational Hierarchical Trees (RHT)}
	\textbf{Require}:
	The constant ${k_1} $ and $ m $, the student's context ${x_t}$ and time $T$.
	
	\textbf{Auxiliary function}: \textbf{Exploration} and \textbf{$\bm{Bound}$ Updating}
	
	\textbf{Initialization}:
	Context sub-hypercubes belonging to  ${\mathcal{P}_T}$\\
	The explored nodes set ${\Gamma ^{{P_t}}}$ = $\{\mathcal{N}_{0,1}^{{P_t}}\}$ \\
	Upper bound of region $\mathcal{N}_{0,1}^{{P_t}}$ over reward $E_{1,i}^{{P_t}} =\infty  $ for $i = 1, 2$.
	\begin{algorithmic}[1]
		\FOR {$t =1,2,...T$}
				\FOR{${d_t} = 0,1,2...{d_X}$}
				\STATE
				Find the context interval in ${d_t}$ dimension
				\ENDFOR
				\STATE
				Get the context sub-hypercube ${P_t}$
		\STATE
		Initialize the current region $\mathcal{N}_{h,i}^{{P_t}} \leftarrow \mathcal{N}_{0,1}^{{P_t}}$
		\STATE
		Build the path set of regions ${\Omega ^{{P_t}}}  \leftarrow \mathcal{N}_{h,i}^{{P_t}}$
		\STATE
		Call \textbf{Exploration ($\Gamma ^{{P_t}}$)}
		\STATE
		Select a course $c_t$ from the region $\mathcal{N}_{h,i}^{{P_t}}$ randomly and recommend to the student $s_t$
		\STATE
		Get the reward ${r_{x_t,c_t}}$
			\FORALL{${P_t} \in {\mathcal{P}_T}$}
			\STATE
			Call \textbf{$\bm{Bound}$ Updating ($\Omega ^{{P_t}}$)}
			\STATE
			$\Omega _{temp}^{{P_t}} \leftarrow {\Omega ^{{P_t}}}$
			\FOR{${\Omega_{temp} ^{{P_t}}} \ne \mathcal{N}_{0,1}^{{P_t}}$ }
			\STATE
			$\mathcal{N}_{h,i}^{{P_t}} \leftarrow $ one leaf of ${\Omega ^{{P_t}}}$
			\STATE
			Refresh the value of $Estimation$ according to (\ref{eq:E})
			\STATE
			Delete the $\mathcal{N}_{h,i}^{{P_t}}$ from ${\Omega_{temp} ^{{P_t}}}$
			\ENDFOR
			\ENDFOR
		\ENDFOR
	\end{algorithmic}
\end{algorithm}

The algorithm is called Reformational Hierarchical Trees (RHT) and the pseudocode is given in Algorithm 1.
We use the explored nodes set $\Gamma^{P_t} = \{ \mathcal{N}_{h_t,i_t}^{{P_t}}| t \in 1, \, 2...T \}$  to denote all the regions whose courses have been recommended in ${P_t}$   and the path set $\Omega^{P_t} =  \{ \mathcal{N}_{h,i}^{{P_t}},\mathcal{N}_{h-1,\left\lceil {{i \over 2}} \right\rceil }^{{P_t}}, \mathcal{N}_{h-2,\left\lceil {{i \over 2^2}} \right\rceil }^{{P_t}}...  \mathcal{N}_{0,1}^{{P_t}} \}$ to show the explored path in ${P_t}$.
Besides, we introduce some new notations, the $Bound$ and the $Estimation$.

We define the $Bound$  $B_{h,i}^{{P_t}}(t)$ as the upper bound reward value of the node $N_{h,i}^{P_t}$ in the depth of $h$ ranked $i$ of the context sub-hypercube ${P_t}$,
\begin{IEEEeqnarray*}{l}
\!\!\!\begin{array}{l}
	  \!B_{h,i}^{{P_t}}(t) \! =\! \hat \mu _{h,i}^{{P_t}}(t) \!+\!\! \sqrt {{k_2}\ln T/{{T}}_{h,i}^{{P_t}}(t)} \! +\! {k_1}{({{{m}}^{{}}})^h}
		\!\! +\!\! {L_{\!X\!}}{({{\sqrt {{d_X}} } \over {{n_T}}})^\alpha }\!\!,
		\IEEEyesnumber \label{eq:B}
	\end{array}
\end{IEEEeqnarray*}
where $k_2$ is a  parameter used to control the exploration-exploitation tradeoff.
And we define the $Estimation$ as the estimated reward value of the node $N_{h,i}^{P_t}$ based on the $Bound$,
\begin{IEEEeqnarray*}{l}
	\!\!\!\!\!\begin{array}{l}
		\! E_{h,i}^{{P_t}}(t) \!=\! \min \! \left\{ B_{h,i}^{{P_t}}(t),max\{ E_{h + 1,2i - 1}^{{P_t}}(t),E_{h + 1,2i}^{{P_t}}(t)\} \right\}.
		\IEEEyesnumber \label{eq:E}
	\end{array}
\end{IEEEeqnarray*}
The role of $E_{h,i}^{{P_t}}(t)$ is to put a tight, optimistic, high-probability upper bound for the reward over the region $\mathcal{N}_{h,i}^{{P_t}}$ of node ${N}_{h,i}^{{P_t}}$ in context sub-hypercube $P_t$ at time $t$.
It's obvious that for the leaf course nodes $N_{h,i}^{P_t}$ we have $E_{h,i}^{P_t}(t)=B_{h,i}^{P_t}(t)$ and for other nodes $N_{h,i}^{P_t}$ we have $E_{h,i}^{P_t}(t) \le B_{h,i}^{P_t}(t)$.

In this algorithm we first find the arrived students' context sub-hypercube $x_t \in {P_t}$ from the context space and replace the original context with the center point ${x^{P_t}}$ in that sub-hypercube ${P_t}$ (line 2-5).
Then the algorithm finds one course region $\mathcal{N}_{h,i}^{P_t}$ whose $E_{h,i}^{{P_t}}(t)$ is highest in the set $\Gamma^{P_t}$ and walks to the region $\mathcal{N}_{h,i}^{P_t}$ with the route $\Omega^{P_t}$, selecting one course $c_t$ from that region and recommending it for the reward $r^{P_t}_{c_t}$ from student $s_t$ (line 7-10). As illustrated in Fig. 4, the algorithm walks upon the nodes with the bold arrow  and the set $\Omega^{P_t} = \{ {N_{0,1}^{P_t}}, {N_{1,2}^{P_t}}, {N_{2,4}^{P_t}}, {N_{3,7}^{P_t}}, {N_{4,13}^{P_t}} \}$, and the node ${N_{4,13}^{P_t}}$ has the highest $Estimation$ value in $\Gamma ^{P_t}$.
When the reward feeds back, the algorithm refreshes $E_{h,i}^{{P_t}}(t)$ of regions of the current tree based on $B_{h,i}^{{P_t}}(t)$ and rewards ${r^{P_t}_{c_t}}(t)$ (line 11-19).
Specifically, the algorithm refreshes the value of $Estimation$ from the leaf nodes to the root node by (\ref{eq:E}) (line 13-18).
Since exploring is a top-down process, after we refresh the upper bound of reward in course regions, we update the $Estimation$ value from bottom to the top based on the $Bound$ with (\ref{eq:E}).
\begin{algorithm}
	\caption{{Exploration}}
	\begin{algorithmic}[1]
		\FORALL {$\mathcal{N}_{h,i}^{{P_t}} \in {\Gamma ^{{P_t}}} $}
		\IF{$E_{h+1,2i-1}^{{P_t}} > E_{h+1,2i}^{{P_t}}$}
		\STATE
		${{Temp }} = 1 $
		\ELSIF{$E_{h+1,2i-1}^{{P_t}} < E_{h+1,2i}^{{P_t}}$}
		\STATE
		${{Temp }} = 0 $
		\ELSE
		\STATE
		${{Temp }} \sim {{ Bernoulli(0}}{{.5)}}$
		\ENDIF
		\STATE
		$\mathcal{N}_{h,i}^{{P_t}} \leftarrow \mathcal{N}_{h + 1,2i - {{Temp}}}^{{P_t}}$
		\STATE
		Select the better region of child node into the path set \\
		${\Omega ^{{P_t}}} \leftarrow {\Omega ^{{P_t}}}  \cup \mathcal{N}_{h,i}^{{P_t}}$
		\ENDFOR
		\STATE
		Add better region of child node into the path set \\
		${\Gamma ^{{P_t}}} \leftarrow \mathcal{N}_{h,i}^{{P_t}} \cup {\Gamma ^{{P_t}}}$
	\end{algorithmic}
\end{algorithm}

\begin{algorithm}
	\caption{${Bound}$ {Updating}}
	\begin{algorithmic}[1]
		\FORALL{$\mathcal{N}_{h,i}^{{P_t}} \in {\Omega ^{{P_t}}}$}
		\STATE
		Refresh selected times ${{T}}_{h,i}^{{P_t}} ++$
		\STATE
		Refresh the average reward according to (\ref{eq:mu})
		\STATE
		Refresh the $Bound$ value on the path according to (\ref{eq:B})
		\ENDFOR
		\STATE
		$E_{h + 1,2i - 1}^{{P_t}} = \infty $, $ E_{h + 1,2i}^{{P_t}} = \infty $
		
	\end{algorithmic}
\end{algorithm}


Algorithm 2 shows the exploration process in RHT.
When we turn to explore new course regions, the model prefers to select the regions with higher $Estimation$ value.
Note that based on (\ref{eq:E}), the parent nodes of the node with the
highest $Estimation$ value also have highest value of the $Estimation$  in their depth, which means for all nodes $N_{h,i}^{P_t} \in \Omega^{P_t}$, we can get that $E_{h,i}^{P_t} = \max\{ E_{h,i'}^{P_t}| 1 \le i' \le 2^h \},$
thus Algorithm 2 can find the node with highest $Estimation$ value.
After the new regions being chosen, they will be taken in the sets $\Gamma ^{{P_t}}$ and $\Omega ^{{P_t}}$ for the next calculation.


\begin{figure}
	\centering
	\includegraphics[scale=.64]{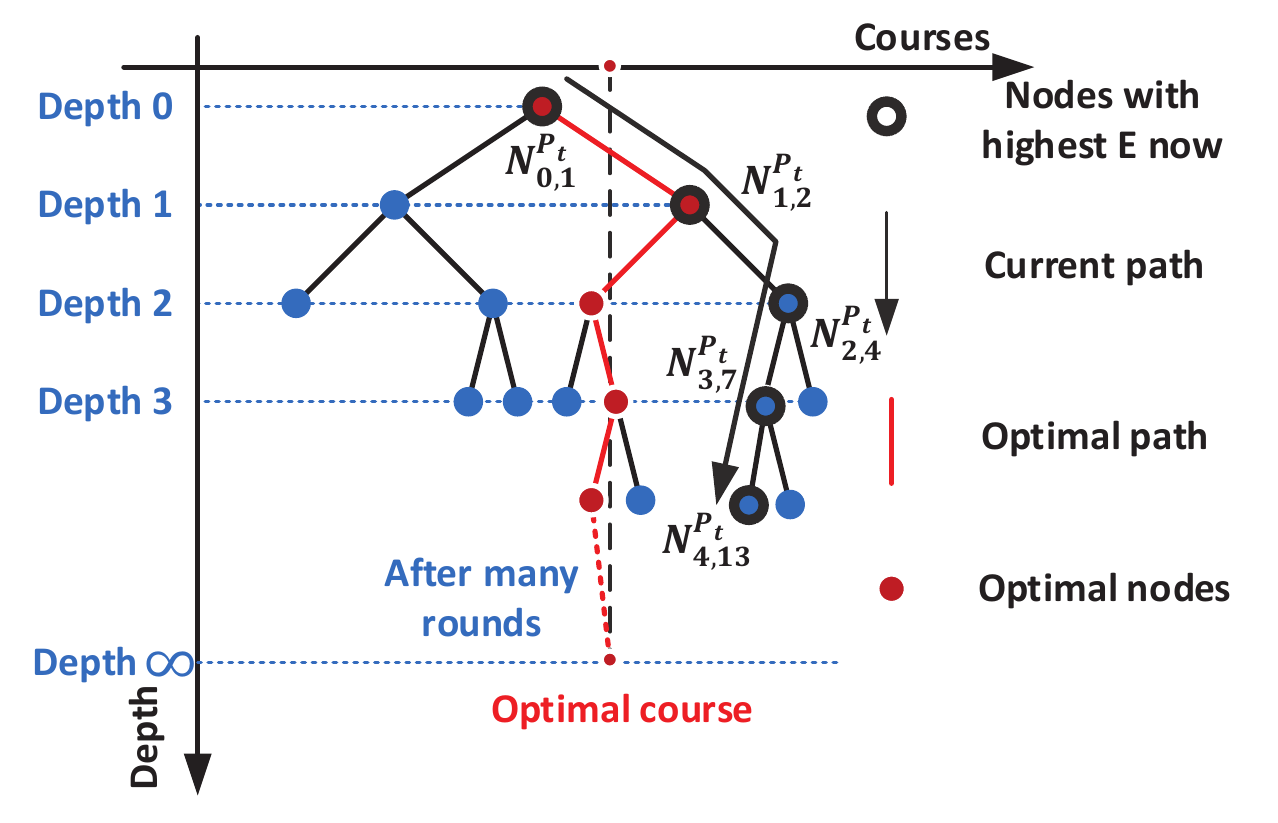}
	\caption{Algorithm Demonstration}
	\label{fig:digraph4}
	\vspace{-.5cm}
\end{figure}

In Algorithm 3,
we define $C(N_{h,i}^{P_t})$ as the set of node $N_{h,i}^{P_t}$ and  its descendants, $$C(N_{h,i}^{P_t}) = N_{h,i}^{P_t} \cup C(N_{h+1,2i-1}^{P_t}) \cup C(N_{h+1,2i}^{P_t}) .$$
And we define $N_{h_t,i_t}^{P_t}$ as the node selected by the algorithm at time $t$.
Then we define $T_{h,i}^{P_t}(t)=\sum\nolimits_t { \mathbb{I} \{N_{h_t,i_t}^{P_t}  \in C({N}_{h,i}^{P_t}) \} } $ as the times that the algorithm has passed by the node $N_{h,i}^{P_t}$, which is equal to the number of selected descendants of $N_{h,i}^{P_t}$ since each node will only be selected once.
We use $B_{h,i}^{{P_t}}(t)$ in (\ref{eq:B})
to indicate the upper bound of highest reward.
The first term  $\hat \mu _{h,i}^{{P_t}}(t) $ is the average rewards, and they come from the students' payoffs as defined
	\begin{IEEEeqnarray*}{l}
		\!\!\!\!\!\begin{array}{l}
			\hat \mu _{h,i}^{{P_t}}(t) = {{(T_{h,i}^{P_t}(t) - 1)\hat \mu _{h,i}^{{P_t}}(t-1) + r_{x_t,c_t}^{{P_t}}}(t) \over {{T_{h,i}^{P_t}(t)}}}.
			\IEEEyesnumber \label{eq:mu}
		\end{array}
	\end{IEEEeqnarray*}
The second one $\sqrt {{k_2}\ln T/{{T}}_{h,i}^{{P_t}}(t)}$ indicates the uncertainty arising from the randomness of the rewards based on the average value.
And the third term $k_1(m)^h$ is the maximum possible variation of the reward over the region $\mathcal{N}_{h,i}^{P_t}$.
As for the last term, since we substitute the sub-hypercube center point for the previous context, we utilize $\max\{D_X(x^{P_t}_i,x^{P_t}_j)\} = \max{ \left\{ {L_X}{{||x_i  -  x_j|}}{{{|}}^\alpha } \right\} } = {L_X}{({{\sqrt {{d_X}} } \over {{n_T}}})^\alpha }$ to denote the deviation in the context sub-hypercube ${P_t}$.

Note that the we only know a part of courses in nodes, uploading   new courses into the cloud would not change the $Estimation$ value and the $Bound $ value (this two is irrelevant with course number), thus the algorithm could hold the past path and explored nodes without recalculating the tree.
Based on this feature, our model can handle the dynamic increasing dataset effectively.
However as for \cite{ACR}, the leaf node is one single course, which means the added courses will change the whole structure of the course tree.
\subsection{Regret Analyze of RHT}
According to the definition of regret in (\ref{eq:regret}), all suboptimal courses which have been selected bring regret.
We consider the regret in one sub-hypercube ${P_t}$ and get the sum of it at last.
Since the regret is the difference between the recommended courses and the best course over reward, we need to define the best course regions at first.
We define the best regions as $\mathcal{N}_{h,{i^*_{h}}}^{{P_t}}$ which contain the best course $c^{P_t*}$ in depth $h$ and optimally ranked $i^*_{h}$ in context sub-hypercube $P_t$ at time $t$.
To illustrate the regret with regions better, we define the best path as
	$\ell _{h,i}^{{P_t}*} =  \{ \mathcal{N}_{h',{i^*_{h'}}}^{{P_t}}| c^{P_t*} \in \mathcal{N}_{h',{i^*_{h'}}}^{{P_t}} \; for \; h'=1, \, 2,...h \}.$

The path is the aggregation of the optimal regions whose depth ranges from $1$ to $h$.
To represent the regret precisely, we need to define the minimum suboptimality gap which indicates the dissimilarity $D_C^{P_t} \left( c^{P_t*}(h,i),c^{P_t*} \right) $ between the optimal course in that region and the overall optimal course $c^{P_t*}$ to better describe the model.

\begin{myDefinition}
	The Minimum Suboptimality Gap is
	\begin{IEEEeqnarray*}{l}
		\!\!\!\!\!\begin{array}{l}
			D_{C(h,i)}^{{P_t}} = f({r^{{P_t}}_{c^{P_t*}}}) - f(r^{P_t}_{c^{P_t*}(h,i)}),
		\end{array}
	\end{IEEEeqnarray*}
	and the Context Gap is
	\begin{IEEEeqnarray*}{l}
		\!\!\!\!\!\begin{array}{l}
			D_{X}^{{P_t}} = \max\{D_X(x^{P_t}_i,x^{P_t}_j)\} = {L_X}{({{\sqrt {{d_X}} } \over {{n_T}}})^\alpha }.
		\end{array}
	\end{IEEEeqnarray*}

\end{myDefinition}

The minimum suboptimality gap of ${{N^{P_t}_{h,i}}}$ is the expected reward defference between overall optimal course and the best one in ${{\mathcal{N}^{P_t}_{h,i}}}$, and the context gap is the difference between the original point and center point in context sub-hypercube ${P_t}$.
As for the context gap, we take the upper bound of it as $\max\{D_X(x^{P_t}_i,x^{P_t}_j)\}$ to bound the regret.

\begin{myAssumption}
	For all courses ${c_j},{c_k} \in {\cal C}$ given the same context vector $x_t$, they satisfy
	\begin{IEEEeqnarray*}{l}
		\!\!\!\!\!\begin{array}{l}
			\;	f(r_{x_t ,c_k } ) \!-\! f(r_{x_t ,c_j } ) \! \le \!  \max \{ f(r_{c^{P_t *} }^{P_t } ) \!-\! f(r_{x_t ,c_k } ),D_C^{P_t } (c_j ,c_k )\} ,
		\end{array}
	\end{IEEEeqnarray*}
	which means
	\begin{IEEEeqnarray*}{l}
		\!\!\!\!\!\begin{array}{l}
			\; f(r_{c^{P_t *} }^{P_t } ) - f(r_{x_t ,c_j } ) \le f(r_{c^{P_t *} }^{P_t } ) - f(r_{x_t ,c_k } )
			\\ \quad\;\;\;\;\;\;\;\;\;\;\;\;\;\;\;\;\;\;\;\;\;\;\;\;\; + \max \{ f(r_{c^{P_t *} }^{P_t } ) - f(r_{x_t ,c_k } ),D_C^{P_t } (c_j ,c_k )\} .
		\end{array}
	\end{IEEEeqnarray*}
\end{myAssumption}
Assumption 3 bounds the difference based on dissimilarity between the optimal course $c^{P_t *}$ and course $c_j$ in context sub-hypercube ${P_t}$ with two terms: (1) the difference between $c^{P_t *}$ and $c_k$; (2) dissimilarity between $c_j$ and $c_k$.
Taking $c_j$, $c_k$ with appropriate values, we could get some useful conclusions presented in the following lemma.
After the definitions and assumptions, we can find a measurement to divide all the regions into two kinds for our following proof.
Based on the Definition 2, we let the set ${\phi ^{{P_t}}}$ to be the $2[{k_1}{({m})^h} + {L_X}{({{\sqrt {{d_X}} } \over {{n_T}}})^\alpha }]$-optimal regions in the depth  $h$,
\begin{IEEEeqnarray*}{l}
	\!\!\!\!\!\begin{array}{l}	
	{\phi ^{{P_t}}} \!\!=\! \bigg\{\! \mathcal{N}_{h,i}^{{P_t}} \big|f({r^{{P_t}}_{c^{P_t*}}}) \!-\! f(r^{P_t}_{c^{P_t*}\!(h,i)}) \!\le\! 2\Big[{k_1}{(m)^h} \!\!+\! {L_{\!X}}{(\!{{\sqrt {{d_X}} } \over {{n_T}}})^\alpha }\Big]  \!\bigg\} \!.
	\end{array}
\end{IEEEeqnarray*}
Note that we call the regions in set ${\phi ^{{P_t}}}$ as optimal regions and those out of it as suboptimal regions.
Besides, we divide the set by depth $h$ which means ${\phi ^{{P_t}}} = \sum\nolimits_h { \phi _h^{P_t } } $, where ${ \phi _h^{P_t }}$ denote the regions in the depth $h$ which are in the set ${\phi ^{{P_t}}} $.

We define the regret when one region is selected above.
Since for every region the algorithm  chooses only once, we can bound the regret after we determine how many regions the algorithm has selected in the recommending process.
Based on definition of $\ell_{h,i}^{P_t*}$ and Definition 2, we assume that the suboptimal regions are divorced from $\ell _{h,i}^{{P_t}*}$ in depth $k$ (in Fig. 4 the depth $k = 2$).
Since we do not know in time $T$ how many times this context sub-hypercube ${P_t}$ has been selected, we use context time $T^{P_t}$ to represent the total times in ${P_t}$.
The sum of $T^{P_t}$ is the total time $\sum\nolimits_{{P_t}} {T^{P_t}}  = T$.

To get the upper bound of the number of suboptimal regions, we introduce Lemma 1 and Lemma 2.
\newtheorem{myLemma}{\textbf{Lemma}}
\begin{myLemma}
	Nodes $ {N}_{h,i}^{{P_t}}$ are suboptimal, and in the depth $k \; (1 \le k \le h - 1)$ the path is out of the best path.
	For any integer $q $, we can get the expect times of the region $\mathcal{N}_{h,i}^{{P_t}}$ and it's descendants in ${P_t}$ are	
\begin{IEEEeqnarray*}{l}
	\!\!\!\!\!\begin{array}{l}	
	\; \mathbb{E}[T_{h,i}^{{P_t}}(T^{P_t}\!)]
	  \!\le\! q \!+\!\!\!\! \sum\limits_{n = q + 1}^{T^{P_t}}\!\!\! {\mathbb{P}\bigg\{\!\! \left[B_{h,i}^{{P_t}}(n) \!>\! f({r^{{P_t}}_{c^{P_t*}}}\!){\;{and}}\;T_{h,i}^{{P_t}}(n) \!>\! q\right]} \\
	\end{array}
\end{IEEEeqnarray*}
\begin{IEEEeqnarray*}{l}
	\!\!\!\!\!\begin{array}{l}	
	  \;\;\;\;\;\;\;\;\;\;\;\; \;\;\;\;\;\;\;\;\;\;\; {{or}}\;\!\! \left[B_{k,{i^{\!*}_{k}}}^{{P_t}}\!(\!n\!) \!\le\! f({r^{{P_t}}_{c^{P_t*}}}\!){\;{f\!or}}\;{{k}} \!\in\! \{ q \!\!+\!\! 1,...,n \!\!-\!\! 1\} \right]\ \!\!\!\! \bigg\}\!.
	\end{array}
\end{IEEEeqnarray*}
\end{myLemma}
\begin{proof}
We assume that the path is out of the best in the depth of $k$.
Since the selected path is out of the optimal path in depth $k$ and the algorithm select the regions with higher $Estimation$ value, we can know that $E_{k ,{i^{*}_{k}}}^{{P_t}}(n) \le E_{k ,i_k }^{{P_t}}(n)$, where the first $Estimation$ value is for the best path region and the second one is for the region selected in the depth of $k$.
	According to (\ref{eq:E}), we can know that $E_{k ,i_k}^{{P_t}}(n) \le E_{k+1,i_{k+1}}^{{P_t}}(n)$,
	then we could get that $E_{k ,{i_{k}^{*}  }}^{{P_t}}(n) \le E_{k ,i_k}^{{P_t}}(n) \le E_{h,i}^{{P_t}}(n) \le B_{h,i}^{{P_t}}(n).$
We define $\{ N_{h_t,i_t}^{P_t} \in C(N_{hi,i}^{P_t}) \}$ as the event that the algorithm passes from the root node by the node $N_{h,i}^{P_t}$.
Obviously, we can get that
$\{ N_{h_t,i_t}^{P_t} \in C(N_{hi,i}^{P_t}) \} \subset \{ B_{h,i}^{{P_t}}(n) \ge E_{k ,{i_{k}^*}}^{{P_t}}(n)\}  .$
So we can bound the time when $C(N_{h,i}^{P_t})$ has been selected as
\begin{IEEEeqnarray*}{l}
		\!\!\begin{array}{l}
        \mathbb{E}[T_{h,i}^{{P_t}}(T^{P_t})] \le \sum\nolimits_{t = 1}^{T^{P_t } } {P\{ B_{h,i}^{P_t } (n) \ge E_{k,i_k^* }^{P_t } (n)\} }  .
        \IEEEyesnumber \label{eq:pp1}
        \end{array}
	   \end{IEEEeqnarray*}
We divide the set $\{ B_{h,i}^{{P_t}}(n) \ge E_{k ,{i_{k}^*}}^{{P_t}}(n)\} $ into
	$ \{ B_{h,i}^{{P_t}}(n) > f({r^{{P_t}}_{c^{P_t*}}})\}  \cup \{ f({r^{{P_t}}_{c^{P_t*}}}) \ge E_{k,{i^*_{{{k} }}}}^{{P_t}}(n)\}. $
	According to the (\ref{eq:E}) once again, we can get
	\begin{IEEEeqnarray*}{l}
		\!\!\begin{array}{l}
			\left\{ f({r^{{P_t}}_{c^{P_t*}}}) \ge E_{k,{i^*_{{{k} }}}}^{{P_t}}(n) \right\}
			\\  \,\,\,\,\, \subset  \left\{ f({r^{{P_t}}_{c^{P_t*}}}) \ge B_{k ,{i^*_{{{k} }}}}^{{P_t}}(n)\right\}
			\cup \left\{ f({r^{{P_t}}_{c^{P_t*}}}) \ge E_{k + 1,{i^*_{{{k + 1}}}}}^{{P_t}}(n)\right\}.
			\IEEEyesnumber \label{eq:Lp1}
		\end{array}
	\end{IEEEeqnarray*}
	From ($\ref{eq:Lp1}$) we  find that the set $\{ f({r^{{P_t}}_{c^{P_t*}}}) \ge E_{k,{i^*_{{{k} }}}}^{{P_t}}(n) \}$ can be divided into two parts, and we notice that $\{ f({r^{{P_t}}_{c^{P_t*}}}) \ge E_{k + 1,{i^*_{{{k + 1}}}}}^{{P_t}}(n)\}$ is similar to $\{ f({r^{{P_t}}_{c^{P_t*}}}) \ge E_{k,{i^*_{{{k} }}}}^{{P_t}}(n) \}$, thus we can keep dividing the set until the depth comes to $k$.
	Hence, we  obtain
	\begin{IEEEeqnarray*}{l}
		\!\!\!\!\!\begin{array}{l}
			\left\{ B_{h,i}^{{P_t}}(n) \ge E_{k ,{i_{k}^*}}^{{P_t}}(n)\right\} \subset \left\{ B_{h,i}^{{P_t}}(n) > f({r^{{P_t}}_{c^{P_t*}}})\right\}
			\\ \quad \quad\quad\quad\quad\quad\quad\quad \,\,\,\,\,\,\, \,\,\,\,\,\, \mathop  \cup \limits_{j = k + 1}^{n - 1} \left\{ f({r^{{P_t}}_{c^{P_t*}}}) \ge B_{j,{i^*_{{{j} \!\! }}}}^{{P_t}}(n)\right\}.
			\IEEEyesnumber \label{eq:Lp1.1}
		\end{array}
	\end{IEEEeqnarray*}
	We introduce an integer $q$ to divide (\ref{eq:pp1}) further. As for any $q $, we have
	\begin{IEEEeqnarray*}{l}
		\!\!\!\!\!\begin{array}{l}
			\; \mathbb{E}\!\left[T_{h,i}^{{P_t}}(T^{P_t}\!)\right]
			\!=\! \sum\limits_{n = 1}^{T^{P_t}} {\mathbb{P}\left\{ B_{h,i}^{{P_t}}(n) \ge E_{k ,{i^*_{{{k } }}}}^{{P_t}}(n),T_{h,i}^{{P_t}}(n) \le q\right\} }
			\\\;\;\;\;\;\;\;\;\;\;\;\;\;\;\;\;\;\;\;\;\;\;\;\;\; +\! \sum\limits_{n = 1}^{T^{P_t}} {\mathbb{P}\left\{ B_{h,i}^{{P_t}}(n) \ge E_{k ,{i^*_{{{k } }}}}^{{P_t}}(n),T_{h,i}^{{P_t}}(n) > q\right\} } 	
			\\\;\;\;\;\;\;\;\;\;\;\;\;\;\;\;\;\;\;\;\;\; \le q \!+\!\!\!\! \sum\limits_{n = q + 1}^{T^{P_t}} \!\!\!\!{\mathbb{P}\Big\{\!\! \left[\!B_{h,i}^{{P_t}}(n) \!>\! f({r^{{P_t}}_{c^{P_t*}}}\!){\;{and}}\;T_{h,i}^{{P_t}}(n) \!\!>\! q  \right]}
			\\\;\;\;\;\;\;\;\;\;\;\;\;\;\;\;\;\;\;\;\;\;\;\;\;\;{{or}}\;\!\! \left[\! B_{j,{i^*_{j}}}^{{P_t}}\!(n) \!\!\le\!\! f({r^{{P_t}}_{c^{P_t*}}}\!){\;{\!f\!or\!}}\;{{j}} \!\! \in \!\! \{  q \!+\! 1,...,\!n \!-\! 1\! \}\! \right]\!\! \Big\}.
		\end{array}
	\end{IEEEeqnarray*}
	In the inequation, we let the event in first term happens all the times so the probability is equal to $1$, and the sum of them is equal to $q$.
	In the second term, since the $T_{h,i}^{{P_t}}(n) > q$, the terms when $n \le q$ are zero and with the help of inequation ($\ref{eq:Lp1.1}$) we can get the conclusion.	
\end{proof}

We determine the threshold of the selected times of the nodes in $C({N}_{h,i}^{{P_t}})$ by Lemma 1.
However, from Lemma 1 we decompose the $\mathbb{E}[T_{h,i}^{P_t}]$ with the sum of events, which means we cannot get the upper bound of $\mathbb{E}[T_{h,i}^{P_t}]$ directly, thus we introduce Lemma 2 to bound $\mathbb{E}[T_{h,i}^{P_t}]$ with the deviation of contexts and courses based on Lemma 1.


\begin{myLemma}
	For the suboptimal regions $\mathcal{N}_{h,i}^{{P_t}} $	,
	if $q$ satisfies
	\begin{IEEEeqnarray*}{l}
		\!\!\!\!\!\begin{array}{l}
			q \ge {{4{k_2}\ln T} \over {{{\left[D_{C(h,i)}^{{P_t}} - {k_1}{{({m})}^h} - {L_X}{{({{\sqrt {{d_X}} } \over {{n_T}}})}^\alpha }\right]}^2}}},
			\IEEEyesnumber \label{eq:L3}
		\end{array}
	\end{IEEEeqnarray*}	
	Then for all $T^{P_t} \ge 1,$ we can get the expected times that node $N_{h,i}^{P_t}$ has been selected as
	\begin{IEEEeqnarray*}{l}
		\!\!\!\!\!\begin{array}{l}
			\mathbb{E}[T_{h,i}^{{P_t}}(T^{P_t})] \le {{4{k_2}\ln T} \over {{{\left[{k_1}{{({m})}^h} + {L_X}{{({{\sqrt {{d_X}} } \over {{n_T}}})}^\alpha }\right]}^2}}} + M,
			\IEEEyesnumber \label{eq:L3.1}
		\end{array}
	\end{IEEEeqnarray*}
where the $M$ is a constant less than 5.
\end{myLemma}
\begin{proof} See appendix A.
\end{proof}
We use the deviation of context and course to represent played times in this lemma.
Practically speaking, we find a upper bound for the  times of suboptimal regions $\mathbb{E}[T_{h,i}^{{P_t}}]$, which means we can determine one region's regret during the process.
But this is not sufficient to bound the whole regret, what we also have to know  is the number of optimal regions.
As mentioned above, we divide the regions into two parts based on the course model as
${\Gamma ^{{P_t}}} = {\phi ^{{P_t}}} \cup {({\phi ^{{P_t}}})^c}$, where $(\bullet)^c$ means the complementary set.
For the convenience, we use the sets of depth to illustrate the region sets
\begin{IEEEeqnarray*}{l}
	\!\!\!\!\!\begin{array}{l}
		{\Gamma ^{{P_t}}} = \left[\sum\nolimits_h {\phi _h^{{P_t}}} \right] \cup \left[\sum\nolimits_h {{{(\phi _h^{{P_t}})}^c}} \right].
		\IEEEyesnumber \label{eq:L3.1p}
	\end{array}
\end{IEEEeqnarray*}
We define the packing number as $\kappa _h^{{P_t}}\left( \cup \mathcal{N}_{h,i}^{{P_t}},Ra\right)$ to show the minimum number of packing balls whose radius is $Ra$ covering optimal regions composed of $\cup \mathcal{N}_{h,i}^{{P_t}}$, where $K$ is the constant of the whole space size, $Ra$ is the packing balls' radius and $d'$ is the dimension of the packing ball.
\begin{myAssumption}
	We assume that there exists a constant $K_0$, that for all the regions of nodes in the depth of $h$, we can get the packing number
	\begin{IEEEeqnarray*}{l}
		\!\!\!\!\!\begin{array}{l}
				\kappa _h^{{P_t}}\left( \cup \mathcal{N}_{h,i}^{{P_t}},Ra\right) = {K_0 \over {{{\left[Ra\right]}^{{d'}}}}}.
			\IEEEyesnumber \label{eq:L3.2}
		\end{array}
	\end{IEEEeqnarray*}
\end{myAssumption}
From this assumption, we could make sure that all the courses in the regions $\{\cup \mathcal{N}_{h,i}^{{P_t}}\}$ can be covered by the packing ball whose radius is $Ra$.
And as for the optimal nodes regions, we could use the packing balls and the radius to bound the regret of them.
\begin{figure}
	\centering
	\includegraphics[scale=.44]{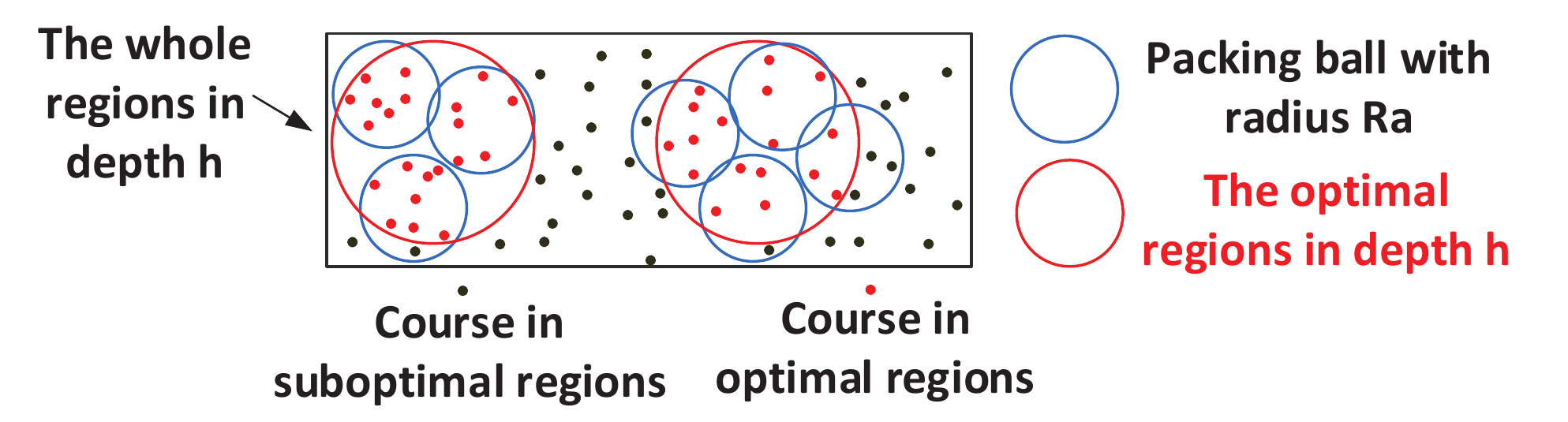}
	\caption{Distributed Storage based on Binary Tree in Cloud}
	\label{fig:digraphpackingball}
\end{figure}
In Fig. 5, we take the dimension of packing ball the same as the course regions as $2$ thus we can illustrate all the courses by dots in black square (plane).
We use the red dot to denote the courses in the optimal regions and black dot to denote the courses in the suboptimal regions in depth $h$.
As shown, we could use the number of packing balls to cover all the courses in the course regions,
which means the number of optimal regions in depth $h$ can be bounded with the number of packing balls with the constant $K_0$ and $\theta$.

With Assumption 4, we introduce Lemma 3 to bound the number of optimal regions in depth $h$ with the number of packing balls.
\begin{myLemma}
	In the same context sub-hypercube ${P_t}$, the number of the $2\left[{k_1}{({m})^h} + {L_X}{({{\sqrt {{d_X}} } \over {{n_T}}})^\alpha }\right]$-optimal regions can be bounded as
	\begin{IEEEeqnarray*}{l}
		\!\!\!\!\!\begin{array}{l}
				\left|\phi _h^{{P_t}}\right| \le K{\left[{k_1}{({m})^h} + {L_X}{({{\sqrt {{d_X}} } \over {{n_T}}})^\alpha }\right]^{{-d_C}}}\!\!.
			\IEEEyesnumber \label{eq:L4}
		\end{array}
	\end{IEEEeqnarray*}
\end{myLemma}
\begin{proof}
	From Assumption 2 we can bound the region with $diam(\mathcal{N}_{h,i}^{P_t}) \ge {{k_1}\over\theta}{({m})^h} $.
	As for context deviation we still use the bound with $ {L_X}{({{\sqrt {{d_X}} } \over {{n_T}}})^\alpha }$.
	Since the course number is can be huge such that we cannot know the data exactly, the dimension of course cannot be determined.
	There exists a constant $d'$,
	\begin{IEEEeqnarray*}{l}
		\!\!\!\!\!\begin{array}{l}
			\left|\phi _h^{{P_t}}\right| \le \kappa _h^{{P_t}}\left( \cup \{ \mathcal{N}_{h,i}^{{P_t}} \in \phi _h^{{P_t}}\}, {{{k_1}\over\theta}}{({m})^h} + {L_X}{({{\sqrt {{d_X}} } \over {{n_T}}})^\alpha }\right)
			\\ \; \; \; \;  \quad  \; \le K_0{\left({{{k_1}\over\theta}}{({m})^h} + {L_X}{({{\sqrt {{d_X}} } \over {{n_T}}})^\alpha }\right)^{ - d'}}.
			\IEEEyesnumber \label{eq:L4.1}
		\end{array}
	\end{IEEEeqnarray*}
Obviously, we know that $\theta>1$ which means we can simplify $K_0{\left({{{k_1}\over\theta}}{({m})^h} + {L_X}{({{\sqrt {{d_X}} } \over {{n_T}}})^\alpha }\right)^{ - d'}}$ further,
	\begin{IEEEeqnarray*}{l}
		\!\!\!\!\!\begin{array}{l}
			K_0{\left({{{\!k_1}\over\theta}}{({m})^h} \!+\! {L_X}{({{\sqrt {{d_X}} } \over {{n_T}}})^\alpha\! }\right)^{ - d'}}
			 \!\!\!\!\!\! \le\!\!\! \ K_0{\left({{{\!k_1}\over\theta}}{({m})^h} \!+\! {L_X\over{\theta}}{({{\sqrt {{d_X}} } \over {{n_T}}})^\alpha \!}\right)^{ \!\!- d'}}
			\\
			\quad\quad\quad\quad\quad\quad\quad\quad\quad\quad\quad\quad\!\!  = K_0 {\theta}^{d'}\!\! {\left({{{\!k_1}}}{({m})^h} \!+\! {L_X}{({{\sqrt {{d_X}} } \over {{n_T}}})^\alpha \!}\right)^{ \!\!- d'}}\!\!\!.
		\end{array}
	\end{IEEEeqnarray*}
	Then we take $K = K_0  {\theta}^{d'}$ to get the conclusion.
The $|\bullet|$ represents the number of elements in the set and we take the minimal $d'$ as the dimension of course ${d_C}$.
	\end{proof}


Since we bound the number of suboptimal regions and optimal regions, we can bound the regret with attained conclusion above.
For simplicity, we divide the regret into three parts according to  $\Gamma ^{{P_t}}=\Gamma _1^{{P_t}}\cup\Gamma _1^{{P_t}}\cup\Gamma _1^{{P_t}}$, where $\mathbb{E}[{R_i}({T})]$  is the expected regret of the set $\Gamma_i ^{{P_t}}$ ($i=1,2,3$).
Then, we can get
\begin{IEEEeqnarray*}{l}
	\!\!\!\!\!\begin{array}{l}
		\mathbb{E}[R({T})] = \mathbb{E}[{R_1}({T})] + \mathbb{E}[{R_2}({T})] + \mathbb{E}[{R_3}({T})],
		\IEEEyesnumber \label{eq:R3}
	\end{array}
\end{IEEEeqnarray*}		
where $\Gamma _1^{{P_t}}$ contains the descendants of $\phi _H^{{P_t}}$ ($H$ is a constant depth to be determined later),
$\Gamma _2^{{P_t}}$ contains the regions $\phi _h^{{P_t}}$ the depth from 1 to $H$
and $\Gamma _3^{{P_t}}$ contains descendants of regions in ${(\phi _h^{{P_t}})^c} (0 \le h \le H)$.
Note that top regions in $\Gamma _3^{{P_t}}$ is the child of regions in $\phi _H^{{P_t}}$.

Due to the fact that $T = \sum {T^{P_t}} $, when all the contexts $x_t $ are in the same context sub-hypercube ${P_t}$, the regret is the smallest.
And we consider the situation that time $T$ is distributed uniformly. Under this condition each context sub-hypercube has the least training data, so the sum of deviation towards course is the largest.
In this extreme situation, all the context sub-hypercube has the same times $T^{P_t}$.
After we know the regret in selecting one region, the times when a region has been selected and the number of chosen regions, we can bound the whole regret in Theorem 1.
\newtheorem{myTheorem}{\textbf{Theorem}}
\begin{myTheorem}
	From the lemma above, regret of RHT is
	\begin{IEEEeqnarray*}{l}
		\!\!\!\!\!\begin{array}{l}
			  \; \mathbb{E}[R(T)] \!=\! O \! \left(\!  {{L_X}^{{d_X  \over  {d_X+\alpha(d_C+3)}  }} T^{{{d_X  + \alpha (d_C  + 2)} \over {d_X  + \alpha (d_C  + 3)}}} (\ln T)^{{\alpha  \over {d_X  + \alpha (d_C  + 3)}}} }  \! \right).
		\end{array}
	\end{IEEEeqnarray*}
\end{myTheorem}
\begin{proof}
	We bound the regret with (\ref{eq:R3}).
	For $\mathbb{E}[{R_1}({T})]$, the regret is generated from the optimal course regions whose courses have been recommended.
	We use the maximum times $T^{P_t}$ to bound the number of optimal regions in $\Gamma_1^{P_t}$.
	Since all the regions in $\Gamma_1^{P_t}$ is optimal,  from Assumption 3 if we take $c_j$ as the worst course in region $\mathcal{N}_{h,i}^{P_t} $ which has the lowest mean reward and $c_k = c^{P_t*}(h,i),$ then we can bound the regret of these nodes as
	\begin{IEEEeqnarray*}{l}
		\!\!\!\!\!\begin{array}{l}
			\mathbb{E}[R_1 (T)] \le \sum\limits_{P_t } {4\left[ {k_1 ({m } )^H  + L_X ({{\sqrt {d_X } } \over {n_T }})^\alpha  } \right]T^{P_t } }
			\\\;\;\;\;\;\;\;\;\;\;\;\;\;\;\; = {4\left[k_1 ({m } )^H  + L_X ({{\sqrt {d_X } } \over {n_T }})^\alpha  \right]T}.
			\IEEEyesnumber \label{eq:T1.2}
		\end{array}
	\end{IEEEeqnarray*}	
	
	As for the second term whose depth is from 1 to $H$, with Lemma 3 and the fact that each regions in $\Gamma_2^{P_t}$ is just played at most once, we can get
	\begin{IEEEeqnarray*}{l}
		\!\!\!\!\!\begin{array}{l}
			 \mathbb{E}[{R_2}({T})]
			 \le \sum\limits_{{P_t}} {\sum\limits_{h = 1}^H {4\left[{k_1}{{({m})}^h} + {L_X}{{({{\sqrt {{d_X}} } \over {{n_T}}})}^\alpha }\right]} \left|\phi _h^{{P_t}}\right|} \\
			 \;\;\;\;\;\;\;\;\;\;\;\;\;\;\; \le {{4K{{({n_T})}^{{d_X}}}} \over {{{\left[{k_1}{{({m})}^H}\right]}^{{d_C}}}}}\sum\limits_{h = 0}^H {4\left[{k_1}{{({m})}^h} + {L_X}{{({{\sqrt {{d_X}} } \over {{n_T}}})}^\alpha }\right]}.
		\end{array}
	\end{IEEEeqnarray*}	
	
	From Lemma 3 we can know the number of optimal regions in depth $h$ are $\left|\phi _h^{{P_t}}\right| \!\le\! K{\left[{k_1}{({m})^h} \!+\! {L_X}{({{\sqrt {{d_X}} } \over {{n_T}}})^\alpha }\right]^{{-d_C}}},$ and the number of the context sub-hypercubes is ${({n_{{T}}})^{{d_X}}}.$
	Thus the last inequation can be derived.
			
When it comes to the last term,
	we notice that the top regions in $\Gamma _3^{{P_t}}$ are the child regions of the regions in $\Gamma _2^{{P_t}}$, since all the regions in $\Gamma _2^{{P_t}}$ is the parent regions of the suboptimal regions.
	And as for the upper bound of course node $k_1(m)^h$, the region of child node is smaller than that of parent node, which means with the depth increasing, the course gap will be smaller than before.
	Hence  we can get that the number of top regions in $\Gamma _3^{{P_t}}$ is less than twice of $\Gamma _2^{{P_t}}$.
	Due to the fact that the child nodes has smaller $diam$ than their parent nodes, we could find that the course deviation of suboptimal region $\mathcal{N}_{h,i}^{P_t}$ can be bounded as $ 4 \left[ {k_1 ({m } )^{h-1}  + L_X ({{\sqrt {d_X } } \over {n_T }})^\alpha  } \right] $.
	And the regret bound is
	\begin{IEEEeqnarray*}{l}
		\!\!\!\!\!\begin{array}{l}
			\; \mathbb{E}[{R_3}({T})]
				 \! \le\! \sum\limits_{{P_t}} \! {\sum\limits_{{{h}} = 1}^H \!{\!4\! \left[{k_1}{{({m})}^{h\!-\!1}} \!\!+\! {L_X}{{({{\sqrt {{d_X}} } \over {{n_T}}})}^\alpha }\right]}  \!\!\sum\limits_{\mathcal{N}_{h,i}^{{P_t}} \in {\Gamma _3^{{P_t}} }} \!\!\!\!\!{\!T_{h,i}^{P_t}(T^{\!P_t\!})} }\\
%
			\,\,\,\,\,\,\quad\quad\quad\; \le \!
			\sum\limits_h \Bigg\{ {{32k_2 K(n_{T} )^{d_{X} } \ln T} \over {\left[ {k_1 (m)^h } \right]^{d_C  + 1} \left[ {k_1 (m)^h  + L_{X} ({{\sqrt {d_X } } \over {n_T }})^{\alpha } } \right]}} \\
		\end{array}
	\end{IEEEeqnarray*}	
\begin{IEEEeqnarray*}{l}
		\!\!\!\!\!\begin{array}{l}
			\quad \quad \quad \quad \quad \quad \quad \quad \quad \quad \; + {{8MK(n_{T} )^{d_{X} } \left[ {k_1 (m)^h  + L_{X} ({{\sqrt {d_X } } \over {n_T }})^{\alpha } } \right]} \over {m\left[ {k_1 (m)^h } \right]^{d_C } }} \Bigg\}. 	
		\end{array}
	\end{IEEEeqnarray*}	
	Note that the bound of $\mathbb{E}[{R_2}({T})]$ is the infinitesimal of higher order of the bound of $\mathbb{E}[{R_3}({T})]$ mathematically, thus we focus more on the first term and the last term since the decisive factors of regret is the first one and last one.
	We notice that with the depth increasing, $\mathbb{E}[{R_1}({T})]$ decreases but $\mathbb{E}[{R_3}({T})]$ increases.
	When we let this two terms to be equal, we can get the regret as follows.
	

	$\mathbb{E}[{R_1}({T})]$ is bounded by
	\begin{IEEEeqnarray*}{l}
		\!\!\!\!\!\begin{array}{l}
			O \left\{ 4\left[{k_1}{({m})^H} + {L_X}{({{\sqrt {{d_X}} } \over {{n_T}}})^\alpha }\right]T\right\} .
			\IEEEyesnumber \label{eq:T1.6}
		\end{array}
	\end{IEEEeqnarray*}

	As for $\mathbb{E}[{R_3}({T})]$, we notice that the constant $M$ is the infinitesimal of higher order of ${{4{k_2}\ln T} \over {{{\left[{k_1}{{({m})}^h} + {L_X}{{({{\sqrt {{d_X}} } \over {{n_T}}})}^\alpha }\right]}^2}}} $, which means we can ignore the influence of the constant $M$.
	Therefore, the bound of  $\mathbb{E}[{R_3}({T})]$ is determined by the first term and it can be shown as
	\begin{IEEEeqnarray*}{l}
		\!\!\!\!\!\begin{array}{l}
			O \left( \sum\limits_h { {{{32k_2 K(n_{T} )^{d_{X} } \ln T} \over {\left[ {k_1 (m)^h } \right]^{d_C  + 1} \left[ {k_1 (m)^h  + L_{X} ({{\sqrt {d_X } } \over {n_T }})^{\alpha } } \right]}}} } \right)
			\\ \quad\quad\quad\quad \quad\quad\quad   \quad\quad\quad\quad\quad\quad  \,\,\,\,\,\, =O \left( {{\ln T{{({n_T})}^{{d_X} }}} \over {{{\left[{k_1}{{({m})}^H}\right]}^{{d_C} + 2}}}}\right) {{ }}.
			\IEEEyesnumber \label{eq:T1.7}
		\end{array}
	\end{IEEEeqnarray*}	
	As for a context sub-hypercube ${P_t}$, all the regions which have been played bring two kinds of regret:
	the regret contributed by context gap ${L_X ({{\sqrt {d_X } } \over {n_T }})^\alpha  }$ and the regret contributed by course region gap ${k_1 ({m } )^H }$.	
	To optimize  the upper bound of regret, we take ${k_1 ({m } )^H  = L_X ({{\sqrt {d_X } } \over {n_T }})^\alpha  }$.
	Under that condition we let  $O(\mathbb{E}[{R_1}({T})]) = O(\mathbb{E}[{R_3}({T})])$ to get
	\begin{IEEEeqnarray*}{l}
		\!\!\!\!\!\begin{array}{l}
			{{ \ln T(n_T )^{d_X } } \over {\left[ {k_1 ({m } )^H } \right]^{d_C  + 2} }} =  {k_1 ({m } )^H} T,
			\IEEEyesnumber \label{eq:T1.8}
		\end{array}
	\end{IEEEeqnarray*}
	where ${n_T} = {\left({{{T}} \over {\ln {T}}}\right)^{{\alpha \over {{d_X} + \alpha({d_C} + 3)}}}}.$
	For the simplicity, we use $\gamma = {{d_X  + \alpha (d_C  + 2)} \over {d_X  + \alpha (d_C  + 3)}}$ and we use the constant $M_2$ to denote the $\mathbb{E}[{R_2}({T})]$ in $\mathbb{E}[{R_3}({T})]$.
	Then we can get the regret as
	\begin{IEEEeqnarray*}{l}
		\!\!\!\!\!\begin{array}{l}
			\; \mathbb{E}[{R}({T})]=8d_X ^{{{\alpha {{[2d}}_X {{ + }}\alpha {{(d}}_C {{ + 3)]}}} \over {{{2[d}}_X {{ + }}\alpha {{(d}}_C {{ + 3)]}}}}} L_X ^{{{d_X } \over {d_X  + \alpha (d_C  + 3)}}} T^\gamma  (\ln T)^{{{1 - }}\gamma } \\
		 \;	+		32k_2 KM_2(d_X )^{{{\alpha (d_C  + 2)(\gamma \! - \! 1)} \over 2}} (L_X )^{(d_C  + 3)\gamma  - (d_C  + 2)}  T^\gamma  (\ln T)^{1\! -\! \gamma } 	
		\IEEEyesnumber \label{eq:R1}
		\end{array}
	\end{IEEEeqnarray*}
	$ \;\;\;\;\;\;\;\;\;\;\;\; =\! O \! \left(\!  {{L_X}^{{d_X  \over  {d_X+\alpha(d_C+3)}  }} T^{{{d_X  + \alpha (d_C  + 2)} \over {d_X  + \alpha (d_C  + 3)}}} (\ln T)^{{\alpha  \over {d_X  + \alpha (d_C  + 3)}}} }  \! \right).$
\end{proof}

\emph{Remark 1:}
From (\ref{eq:R1}) we can make sure $\mathop {\lim }\nolimits_{T \to \infty } {{\mathbb{E}[R(T)]} \over T} = 0$, which means the algorithm can find the optimal courses for the students finally.
Note that the tree exists actually, we store the tree in the cloud and during the recommending process.
Since the dataset is fairly large in the future, using the distributed storage method to solve storage problems is inescapable.


\section{Distributively  Stored Course  Tree}

\subsection{Distributed Algorithm for Multiple Course Storage}
In practice, there are many MOOC platforms e.g. Coursera, edX, Udacity, and the course resources are stored in their respective databases.
Thus course recommendation towards heterogeneous sources in the course cloud needs to be handled by a system that supports distributed-connected storage nodes, where the storage nodes are in the same cloud with different zones.
In this section, we turn to present a new algorithm called Distributed Storage Reformational Hierarchical Trees (DSRHT), which can handle the heterogeneous sources of course datasets and improve the storage condition by mapping them into distributed units in the course cloud.

We denote the distributed storage units whose number is $d$ as $\mathcal{Z} = \{ Z_1, \, Z_2,...Z_d \}$, where $Z_i$ could be a MOOC learning platform.
We bound the number of distributed units  $Z_d$ with $2^{z-1} < d  \le 2^z $ to fit with the binary tree mode,
where $z$ is the depth of the tree and $2^z$ is the number of regions in that depth.
Note that the number of distributed units is determined by the practical situation, thus in every context sub-hypercube $	{P_t}$ the number of elements in set $\mathcal{Z}$ is the same as $d$.
Since $Z_d$ is not always equal to $2^z$, we let the storage units whose regions are empty $ \mathcal{Z}_{\emptyset} = \{ Z_{d+1}, \, Z_{d+2},... Z_{2^z} \} $ be the virtual nodes, which means there is no course in that distributed units $\{ {Z}_{j} = \emptyset | j=d+1, \, d+2,...2^z \,  \}$ for any context sub-hypercube.
Fig. 6 illustrates the condition when there are 3 storage platforms (Coursera, edX and Udacity).
We can get the number of distributed units as $d=3$ and the depth is $z=2 \; (2^1 \le 3 \le 2^2)$, and the set $\mathcal{Z} = \{ Z_1, Z_2, Z_3 \}$ and the set $\mathcal{Z}_\emptyset = \{Z_4\}$.

\begin{figure}
	\centering
	\includegraphics[scale=.55]{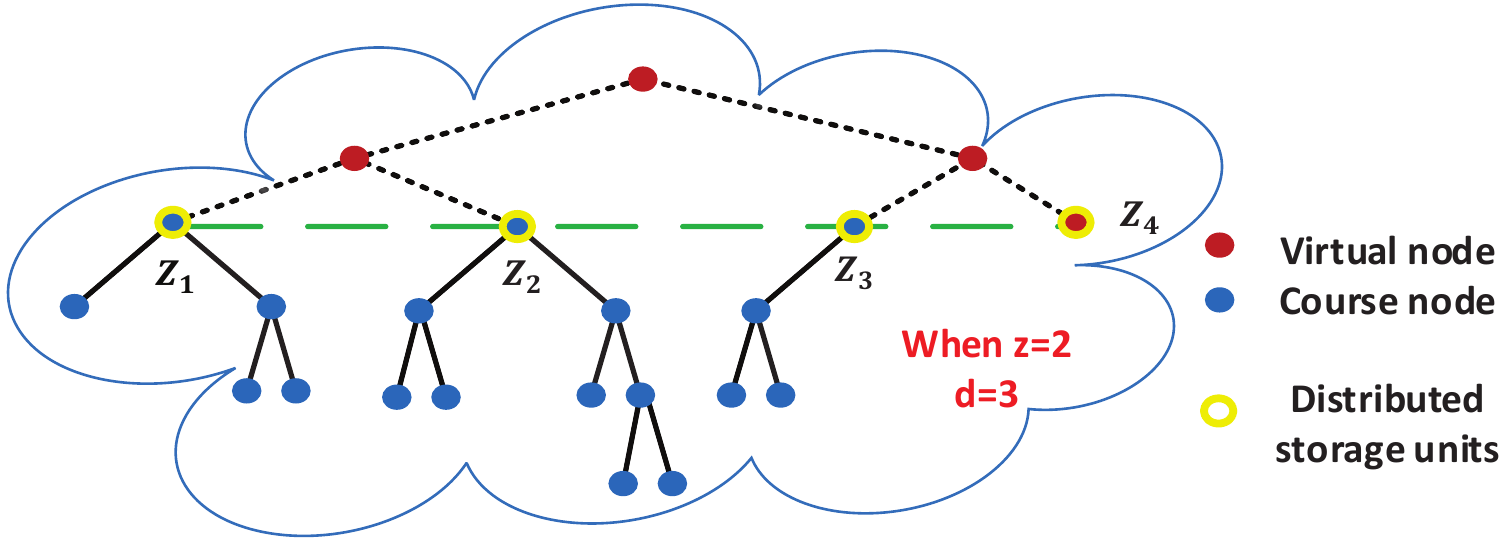}
	\caption{Distributed Storage based on Binary Tree in Cloud}
	\label{fig:digraphc}
\end{figure}

\begin{algorithm}
	\caption{Distributed Course Recommendation Tree}
	\textbf{Require}:
	The constants ${k_1} $ and $m$, the parameter of the storage unit $z$,
	the student's context ${x_t}$ and time $T	$.
	
	\textbf{Auxiliary function}: \textbf{Exploration} and \textbf{$\bm{Bound}$ Updating}
	
	\textbf{Initialization}:
	For all context sub-hypercubes belonging to  ${\mathcal{P}_T}$\\
	${\Gamma ^{{P_t}}} = \{ \mathcal{N}_{z,1}^{{P_t}},\mathcal{N}_{z,2}^{{P_t}}...\mathcal{N}_{z,{2^z}}^{{P_t}}\} $ \\ $E_{z,i}^{{P_t}} = \infty $ ${{for}} \;i  =  1, 2...{{{2}}^z}$
	\begin{algorithmic}[1]
		\FOR {t =1,2,...T}
		\FOR{${d_t} = 0,1,2...{d_X}$}
		\STATE
		Find the context interval in ${d_t}$ dimension
		\ENDFOR
		\STATE
		Get the context sub-hypercube ${P_t}$
		\STATE
		${x_t} \leftarrow $ center point of $ {P_t}$
		\FOR{j=1,2...$2^z-1$}
		\IF{$\mathcal{N}_{z,j}^{{P_t}} < \mathcal{N}_{z,j + 1}^{{P_t}}$}
		\STATE
		$\mathcal{N}_{z,j}^{{P_t}} = \mathcal{N}_{z,j + 1}^{{P_t}}$
		\ENDIF
		\ENDFOR
		\STATE
		$\mathcal{N}_{h,i}^{{P_t}} \leftarrow \mathcal{N}_{z,j}^{{P_t}}$, ${\Omega ^{{P_t}}}  \leftarrow \mathcal{N}_{h,i}^{{P_t}}$
		
		\STATE
		Same to Algorithm 1 from line 8 to line 19
		\ENDFOR
	\end{algorithmic}
\end{algorithm}
In Algorithm 4, we  still find the context sub-hypercube at first (line 2-6).
Then since there are $2^z$ distributed units, we first identify these top regions (line 7-12).
Based on the attained information, the algorithm can start to find the course by utilizing the $Bound$ and $Estimation$ the same as Algorithm 1 (line 13).
For the virtual nodes, we set the $Bound$ value of them as $0$.
As for the tree partition, the difference is that we leave the course regions whose depth is less than  $z$ out to cut down the storage cost.
In the complexity section we will prove that the storage can be bounded sublinearly under the optimal condition.

\subsection{Regret Analyze of DSRHT}



In this subsection we prove the regret result in DSRHT can be bounded sublinearly.
Now, again,  we divide the regions contrast to get the regret upper bound separately by
${\Gamma ^{{P_t}}} = \Gamma _1^{{P_t}} + \Gamma _2^{{P_t}} + \Gamma _3^{{P_t}} + \Gamma _4^{{P_t}},$ where $\mathbb{E}[{R_i}({T})]$  is the expected regret of the set $\Gamma_i ^{{P_t}}$ ($i=1,2,3,4$).
$\Gamma _1^{{P_t}}$ means the regions and their descendants in set $\phi _H^{{P_t}}$ whose depth is $H(H>z)$; $\Gamma _2^{{P_t}}$ is the set whose regions are in set $\phi _h^{{P_t}}\;(z < h \le H)$; $\Gamma _3^{{P_t}}$ contains the regions and their descendants in set ${(\phi _h^{{P_t}})^c}(z < h \le H)$;
 and for $\Gamma _4^{{P_t}}$, they are the regions at depth $z$ which will be selected twice each based on the Algorithm 1.
The depth $H\; (z < H )$ is a constant to be selected later.

\begin{myTheorem}
	The regret of the distributively stored algorithm is
	\begin{IEEEeqnarray*}{l}
		\!\!\!\!\!\begin{array}{l}
			\; \mathbb{E}[R(T)] \!=\! O \! \left(\!  {{L_X}^{{d_X  \over  {d_X+\alpha(d_C+3)}  }} T^{{{d_X  + \alpha (d_C  + 2)} \over {d_X  + \alpha (d_C  + 3)}}} (\ln T)^{{\alpha  \over {d_X  + \alpha (d_C  + 3)}}} }  \! \right),
		\end{array}
	\end{IEEEeqnarray*}
if the number of distributed units satisfies
\begin{IEEEeqnarray*}{l}
	\!\!\!\!\!\begin{array}{l}
		{d} \le 2^z \le {\left({T \over {\ln T}}\right)^{{{{d_X} + \alpha {d_C}{{ }}} \over {{d_X} +\alpha({d_C} + 3) }}}}.
		\IEEEyesnumber \label{eq:A4}
	\end{array}
\end{IEEEeqnarray*}
\end{myTheorem}

\begin{proof} (Sketch) Detailed proof is given in Appendix B.
For the first third term, the regret upper bound is the less than the result in Theorem 1, since the regret of node $N_{h,i}^{P_t}$ will be larger as far as the increasing depth $h$.

	When it comes to the fourth term, we notice that since the depth of $z$ is bounded, and the worst situation happens when the number of distributed units is the maximum ($2^z$).
	
	\begin{IEEEeqnarray*}{l}
		\!\!\!\!\!\begin{array}{l}
			\mathbb{E}[{R_4}({T})] \le ({2^z} - 1)\left\{ {{4{k_2}{{\ln T}}} \over {{{\left[{k_1}{{({m})}^z} + {L_X}{{({{\sqrt {{d_X}} } \over {{n_T}}})}^\alpha }\right]}^2}}} + M\right\}\\
			\;\;\;\;\;\;\;\;\;\;\;\,\quad \le{\left({T \over {\ln T}}\right)^{\!\!{{{d_X} + \alpha {d_C}{{ }}} \over {{d_X} \!+\! \alpha({d_C} \!+\! 3) }}}}  \!\!\!  \left\{ \!\! {{4{k_2}{{\ln T}}} \over {{{\left[{k_1}{{({m})}^z} + {L_X}{{({{\sqrt {{d_X}} } \over {{n_T}}})}^\alpha }\right]}^2}}} \!+\!  M \!\! \right\}.
			\IEEEyesnumber \label{eq:T2.5}
		\end{array}
	\end{IEEEeqnarray*}
	For the value of $n_T$ determined by the first third term ${n_T} = {\left({{{T}} \over {\ln {T}}}\right)^{{\alpha \over {{d_X} + \alpha({d_C} + 3)}}}}.$
	we have
	\begin{IEEEeqnarray*}{l}
		\!\!\!\!\!\begin{array}{l}
			{{  }}\mathbb{E}[{R_4}(T)]{{  }} = O {\left({\left({T \over {\ln T}}\right)^{{{{d_X} + \alpha{d_C}{{ }}} \over {{d_X} + \alpha({d_C} + 3)}}}}{\mathop{ \ln T}\nolimits} {({n_T}{{)}}^2}\right)}
			\\     \;\;\;\;\;\;\;\;\;\;\;\;\;\;\; = O \left({T^{{{{d_X} + \alpha({{d_C} +2}}) \over {{d_X} + \alpha{({d_C} + 3)}}}}}{(\ln T)^{{\alpha \over {{d_X} + \alpha({d_C} + 3)}}}}\right).
			\IEEEyesnumber \label{eq:T2.6}
		\end{array}
	\end{IEEEeqnarray*}
	
	From Theorem 1, we minimize the regret by making context gap and course region gap equal too,
	i.e., ${{k_1}{{({m})}^H} = {L_X}{{({{\sqrt {{d_X}} } \over {{n_T}}})}^\alpha }}.$
	For the simplicity we take the constant ${k_2} = 2$, and the \emph{slicing number} can be derived  by setting $O(\mathbb{E}[{R_1}(T)]) = O(\mathbb{E}[{R_3}(T)])$ as  ${n_T} = {\left({{{T}} \over {\ln {T}}}\right)^{{\alpha \over {{d_X} + \alpha({d_C} + 3)}}}}.$
\end{proof}

\emph{Remark 2:}
Note that if there is only one distributed unit ($z=0$), the regret $\mathbb{E}[{R_4}(T)] = 0$, thus we can get the conclusion of Theorem 1.
Compared  to the RHT algorithm, we notice that the regret upper bound is the same.
Since this algorithm starts at the depth of $z$, it need to explore all the nodes in depth $z$ first.
Thus it performs not as well as RHT in the beginning.
However, the algorithm can fit the practical problem better since there are many MOOC platforms in practice.

\section{Storage Complexity}

\begin{table*}
	\caption{Theoretical Comparison}
	\begin{center}
		\scalebox{1.1} {
			\begin{tabular}{c|c|c|c|c|c}
				\toprule
				Algorithm & Context & Big data-oriented & Time Complexity & Space Complexity  & Regret \\
				\midrule
				ACR\cite{ACR} & Yes & No & $O \left(T^2  + K_E T\right)	$ & $O \left(\sum\nolimits_{l = 0}^E {K_l  + T}\right) $ & $O \left(T^{{{d_I  + d_C  + 1} \over {d_I  + d_C  + 2}}} \ln T\right)	$ \\
				HCT\cite{HCT} & No & Yes & $O (T\ln T)$ & $O \left(T^{{d \over {d+ 2}}} (\ln T)^{{2 \over {d  + 2}}}\right)	$ & $O \left(T^{{{d + 1} \over {d  + 2}}} (\ln T)^{{1 \over {d  + 2}}} \right)	$ \\
				RHT & Yes & Yes & $O (T\ln T)$	 & $O (T)$ & $O \left(T^{{{d_X  + d_C  + 2} \over {d_X  + d_C  + 3}}} (\ln T)^{{1 \over {d_X  + d_C  + 3}}}\right) $
				\\
				DSRHT & Yes & Yes & $O (T\ln T)$	 & $O \left(T - 2^z \right)$ & $O \left(T^{{{d_X  + d_C  + 2} \over {d_X  + d_C  + 3}}} (\ln T)^{{1 \over {d_X  + d_C  + 3}}}\right) $	\\ 	
				\bottomrule
			\end{tabular} }
		\end{center}
	\end{table*}


The storage problem has been existing in big data analytics for a long time, so how to use the distributed storage scheme  to handle the problem matters a lot.
In this section, we analyze the two algorithms' space complexity mathematically.
We use ${{S}}(T)$ to represent the storage space complexity.
For RHT algorithm, since it explores one region in one round, it's obvious to know the space complexity is linear 	$ \mathbb{E}[S(T)]=O(T) $.
\begin{myTheorem}
	In the optimal condition, we take the number of storage units satisfied ${2^z} = {\left({T \over {\ln T}}\right)^{{{{d_X} + \alpha{d_C}{{ }}} \over {{d_X} + \alpha({d_C} + 3)}}}}$, then we can get the space complexity
	\begin{IEEEeqnarray*}{l}
		\!\!\!\!\!\begin{array}{l}
			\mathbb{E}[S(T)]
			 \!=\! O \!\left(\!T^{{{d_X  + \alpha d_C  } \over {d_X  + \alpha (d_C  + 3)}}} \!\left(\!T^{{3\alpha \over {d_X  + \alpha(d_C  + 3)}}}  \!-\! (\ln T)^{{{d_X  + \alpha(d_C  + 3)} \over {d_X  + \alpha d_C  }}} \!\right)\!\right)\!.
		\end{array}
	\end{IEEEeqnarray*}
\end{myTheorem}

\begin{proof}
	Every round $t$ has to explore a new leaf region.
	To get the optimal result, we suppose the depth is as deepest as we can choose $  z = \left\lfloor {{{{{d_X  + \alpha d_C } \over {d_X  + \alpha(d_C  + 3)}}\ln \left({T \over {\ln T}}\right)} \over {\ln 2}}} \right\rfloor .$
	Under the condition that $t< 2^{z+1}$, we have
	$S_1 (T) \le 2^z  = \left({T \over {\ln T}}\right)^{{{d_X  + \alpha d_C  } \over {d_X  + \alpha (d_C  + 2) }}} ,$
	when the time $	t \ge 2^{z+1} $, after one round there is one unplayed region being selected,
	so the second part is $S_2 (T) \le T - 2^{z + 1}  = T - 2\left({T \over {\ln T}}\right)^{{{d_X  + \alpha d_C  } \over {d_X  + \alpha (d_C  + 2)}}} .	$
	Thus we can get the storage complexity
	\begin{IEEEeqnarray*}{l}
		\!\!\!\!\!\begin{array}{l}
			\mathbb{E}[S(T)] = O \left(T - \left({T \over {\ln T}}\right)^{{{d_X  + \alpha d_C  } \over {d_X  + \alpha (d_C  + 2)}}} \right).
			\IEEEyesnumber \label{eq:T3.1}
		\end{array}
	\end{IEEEeqnarray*}
\end{proof}

\emph{Remark 3:}
Since the value of $z$ is changeable, appropriate value can make the space complexity sublinear.
From ($\ref{eq:T3.1}$), if the data dimension is fairly large, the space complexity will be relative small.
However, the large database and tremendous distributed units will make the algorithm learning too slow.
Thus taking an appropriate parameter is crucial.

Besides, we compare our algorithms with some similar works which all use the tree partition.
%
%
In table \Rmnum{1} we categorize these algorithms based on the following characteristics: context-awareness, big data-oriented, time complexity, space complexity and regret.
As for the context-awareness and big data-oriented, our two algorithms both take them into consideration, and ACR\cite{ACR} and HCT\cite{HCT} only take one respect each.
For the time complexity,  we can find that the ACR\cite{ACR} is polynomial in $T$ with $O \left(T^2  + K_E T\right)$ but others are linear with time $O \left(T\ln T \right) $.
When it comes to space complexity, our algorithm RHT and algorithm ACR\cite{ACR} can bound it linearly, and the HCT\cite{HCT} reduces it to sublinear. For our DSRHT, we can also realize the sublinear space complexity under the optimal condition.
The four algorithms all realize the sublinear regret,
and our two algorithms can bound the regret with $O \left(T^{{{d_X  + d_C  + 2} \over {d_X  + d_C  + 3}}} (\ln T)^{{1 \over {d_X  + d_C  + 3}}}\right)  $ by setting $\alpha = 1$ to make sure fair comparison with ACR\cite{ACR} and HCT\cite{HCT}.
To sum up, our algorithms not only consider the context-awareness but also  are big data-oriented. Besides, their time complexity and space complexity are promising.


\section{Numerical Results}

In this section,
we present:
(1) the source of data-set;
(2) the sum of regret are sublinear and the average regret converges to $0$ finally;
(3) we compare the regret bounds of our algorithms with other similar works;
(4) distributed storage method can reduce the space complexity.
Fig. 5 illustrates the MOOC operation pattern in edX\cite{edX}.
The right side is the teaching window and learning resources, and the left includes lessons content, homepage, forums and other function options.
\subsection{Description of the Database}
We take the database which contains feedback information and course details from the edX\cite{edX} and the intermediary website of MOOC\cite{MOOC}.
In those platforms, the context dimensions contain nationality, gender, age and the highest education level, therefore we take $d_X = 4$.
As for the course dimensions, they comprise starting time, language, professional level, provided school, course and program proportion, whether it's self-paced, subordinative subject etc.
Thus we take the course dimension as $10$.
For the feedback system, we can acquire reward information from review plates and forums.
Thoroughly, the reward is produced from two aspects, which are the marking system and the comments from forums.

For the users, when a novel field comes into vogue, tremendous people will get access to this field in seconds.
The data we get include $2\times10^5$ students using MOOC in those platforms,
and the average number of courses the students comment is around 30.
As for our algorithm, it focuses on the group of students in the same context sub-hypercube rather than individuals.
Thus,  when in the next time users come  with context information and historical records, we just treat them as the new training data without distinguishing them.
However the number of users is limited, even if generating a course is time-costing, the number of courses is unlimited and education runs through the development of human being.
Our algorithm pays more attention to the future highly inflated MOOC curriculum resources, and existing data bank is not tremendous enough to demonstrate the superiority of our algorithm since MOOC is a new field in education.

We find 11352 courses from those platforms including plenty of finished courses.
The number of courses doubles every year.
Based on the trend, the quantity will be more than forty thousand times within 20 years.
To give consideration to both accuracy and scale of sources of data, we copy the original sources to forty five thousand times to satisfy the number requirements.
Thus we extend the 11352 course data  to around $5 \times 10^8$ to simulate future explosive data size of courses in 2030.

\subsection{Experimental Setup}


 \begin{figure}
 	\centering
 	\includegraphics[scale=.145]{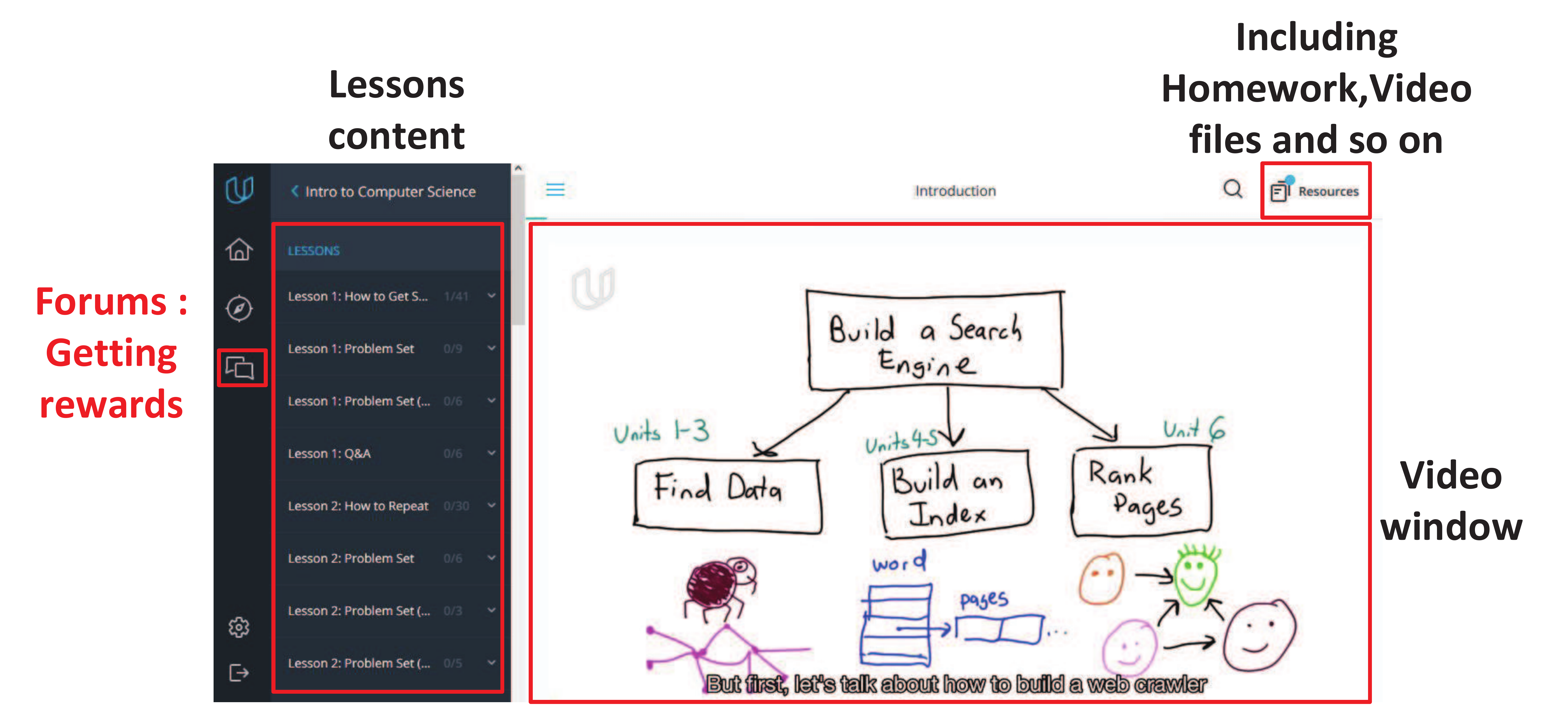}
 	\caption{MOOC Learning Model}
 	\label{fig:digraph0}
 \end{figure}

As for our algorithm, the final training number of data is over $6\times 10^6$ and the number of courses is about $5 \times 10^8$.
Note that we focus more on the comparison rather than showing the superiorities of our algorithms, thus we take the statistic course data to better illustrate the comparing effect.
The works are introduced as follows.


\begin{itemize}
	\item Adaptive Clustering Recommendation Algorithm (ACR)\cite{ACR}: The algorithm injects contextual factors capable of adapting to more students, however, when the course database is fairly large, ergodic process in this model cannot handle the dataset well.
	\item High Confidence Tree algorithm (HCT)\cite{HCT}: The algorithm supports unlimited dataset however large it is, but there is only one student for the recommendation model since it does not take context into consideration.
	\item We consider both the scale of courses and users' context, thus our model can better suit future MOOC situation. In DSRHT we sacrifice some immediate interests to get better long-term performance.
\end{itemize}
To verify the conclusions practically, we divide the experiment into following three steps:

\subsubsection{Step 1.}
In this step we compare our RHT algorithm with the two previous works which are ACR\cite{ACR} and HCT\cite{HCT} with different size of training data.
We input over $6\times 10^6$ training data including context information and feedback records in the reward space mentioned in the section of database description into the three models, and then the models will start to recommend the courses stored in the cloud.
In consideration of HCT not supporting context, we normalize all the context information to the same (center point of unit context hypercube).
Since the reward distribution is stochastic, we simulate 10 times to get the average values where the interfere of random factor is restrained.
Then the two regret tendency diagrams are plotted to evaluate algorithms performances.


\subsubsection{Step 2.}
We use the DSRHT algorithm to simulate the results.
The RHT algorithm can be seemed as degraded DSRHT with $z = 0$, and we compare the DSRHT algorithm with different parameters $z$.
Without loss of generality, we take $z = 0,$ $ z = 10 $ and $z = \left\lfloor {{{{{d_X  + \alpha d_C } \over {d_X  + \alpha(d_C  + 3)}}\ln \left({T \over {\ln T}}\right)} \over {\ln 2}}} \right\rfloor \approx 20$.
Then we plot the regret and $z$ diagram to analyze the constant optimal parameter.


\begin{figure}
	\centering
	\includegraphics[scale=.568]{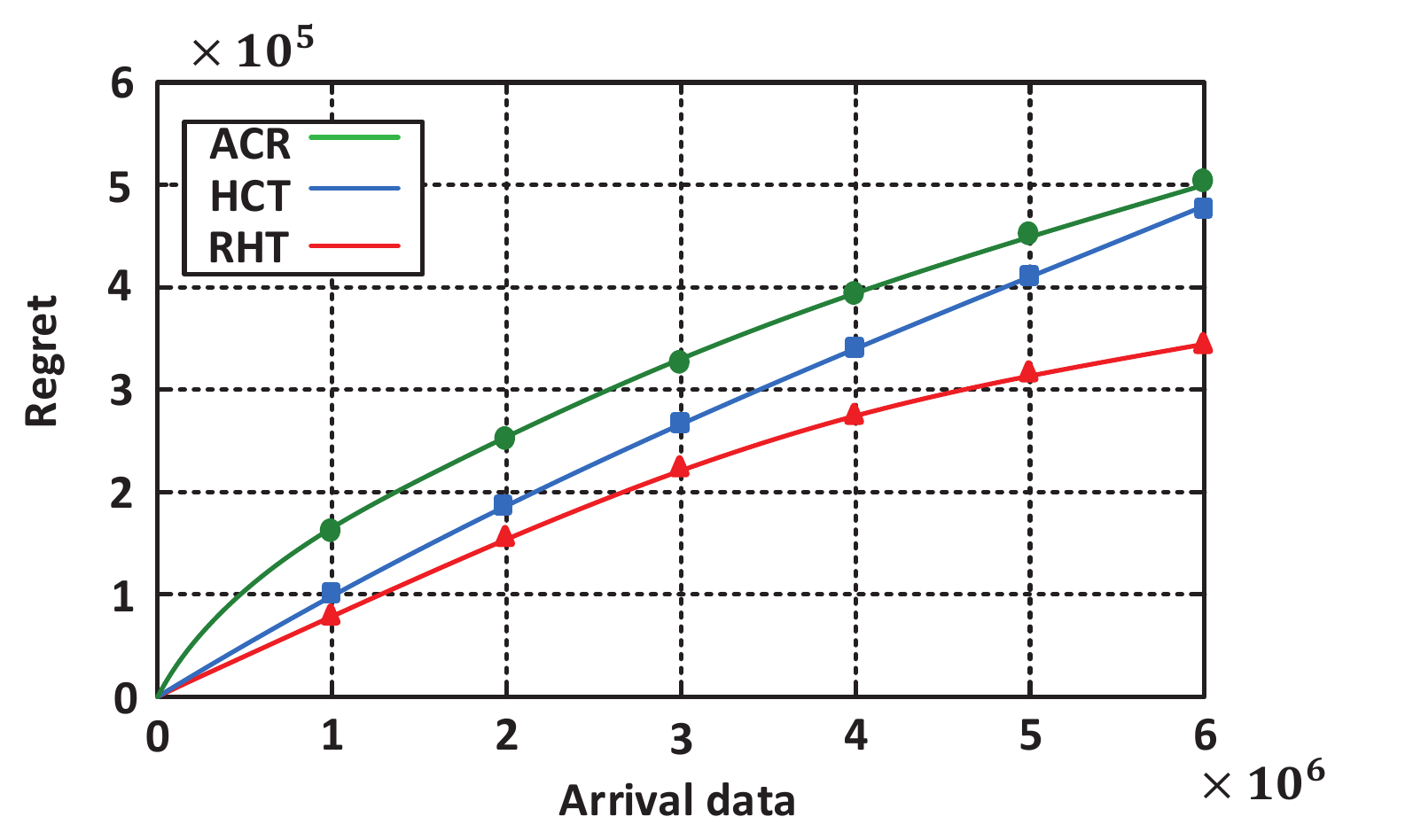}
	\caption{Comparison of Regret (RHT)}
	\label{fig:digraph5}
\end{figure}

\begin{figure}
	\centering
	\includegraphics[scale=.568]{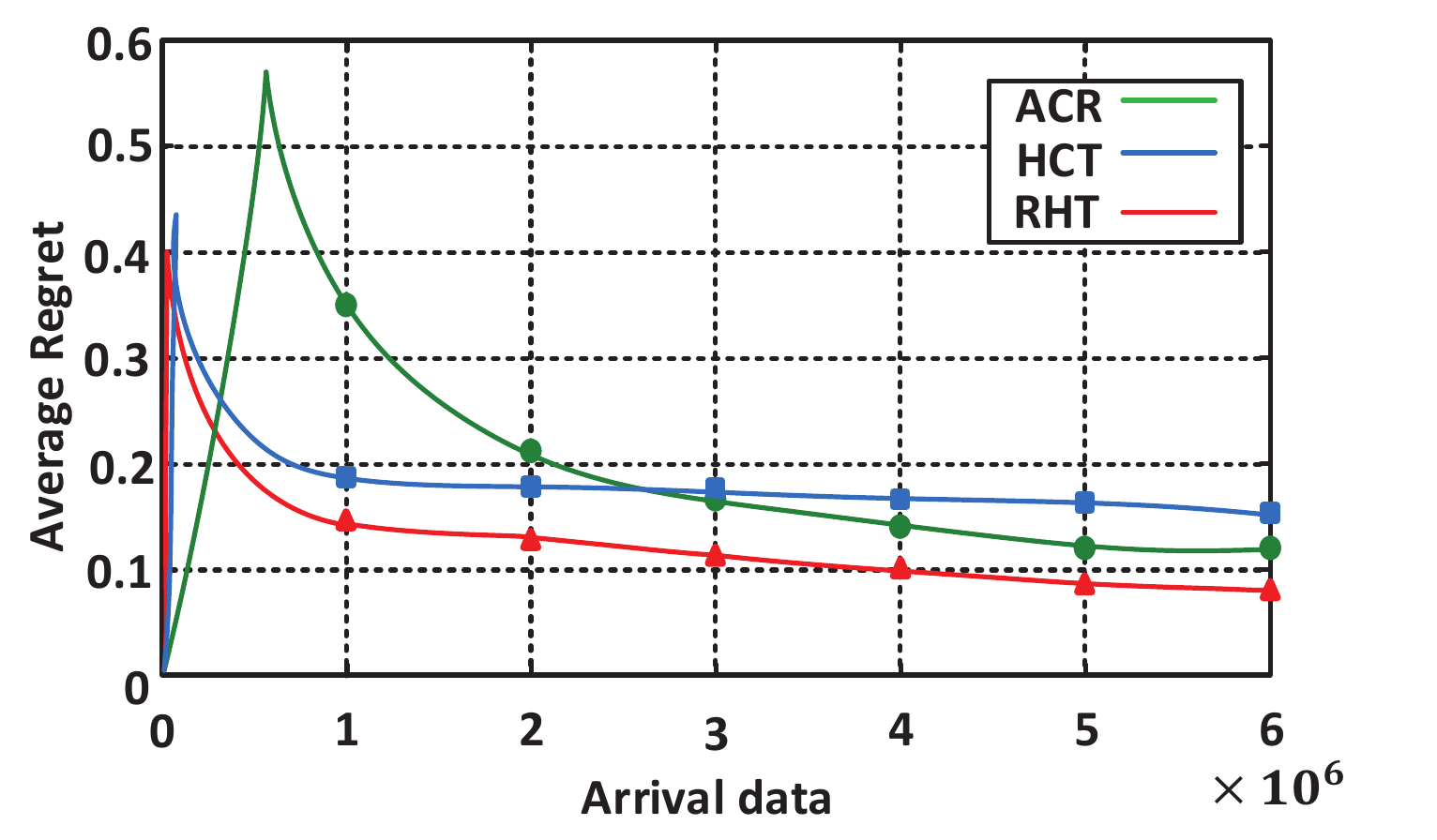}
	\caption{Comparison of Average Regret (RHT)}
	\label{fig:digraph6}
\end{figure}


\begin{table}
	\caption{Average Accuracies of RHT}	
	\begin{center}
		\begin{tabular}{c|cccc|}
			\toprule
			\diagbox{Number $\times 10^6 $}{Algorithm} & ACR\cite{ACR} & HCT\cite{HCT} & RHT  \\
			\midrule
			1 & 65.43\% & 81.02\% &  85.34\%  \\
			2 & 78.62\% & 82.13\% &  87.62\%  \\
			3 & 83.23\% & 82.76\% &  89.92\%  \\
			4 & 86.28\% & 83.01\% &  90.45\%  \\
			5 & 88.19\% & 83.22\% &  91.09\%  \\
			6 & 88.79\% & 83.98\% &  91.87\%  \\
			\bottomrule
		\end{tabular}
	\end{center}
\end{table}

\subsubsection{Step 3.}
We record the storage data to analyze the space complexity of those four algorithms.
First we upload 517.68 TB indexing information of courses to our university high performance
 computing platform, whose GPU reaches to 18.46 TFlops and SSD cache is 1.25 TB. Then, we implement and perform the four algorithms successively.
In the process of training, we record the regret for six times.
And in the end of training, we record the space usage of the tree which represent the training cost.
As for the DSRHT, we use the virtual partitions in school servers to simulate the distributively stored course data.
Specifically, we reupload the course data to the school servers in 1024 virtual partitions, and then perform the DSRHT algorithm.


\subsection{Results and Analysis}

We analyze our algorithm from two different angles: Comparing with other two works and comparing with itself with different parameter $z$.
In each direction, we compare the regret first, and analyze the average regret.
And then we discuss the accuracies based on the average regret.
At last we will compare the storage conditions from different algorithms.

\begin{figure}
	\centering
	\includegraphics[scale=.568]{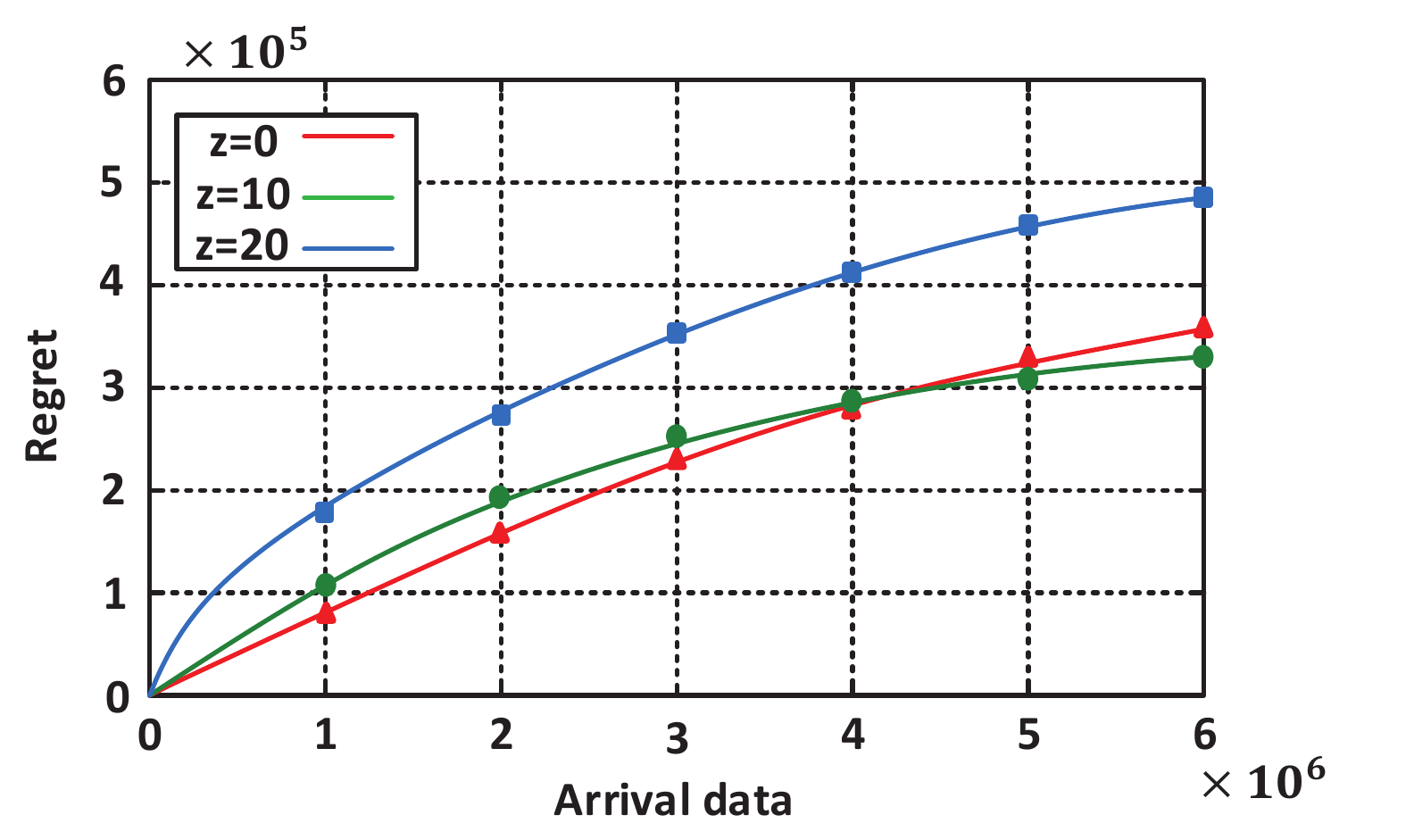}
	\caption{Comparison of Regret with Different  $z$}
	\label{fig:digraph8}
\end{figure}

\begin{figure}
	\centering
	\includegraphics[scale=.568	]{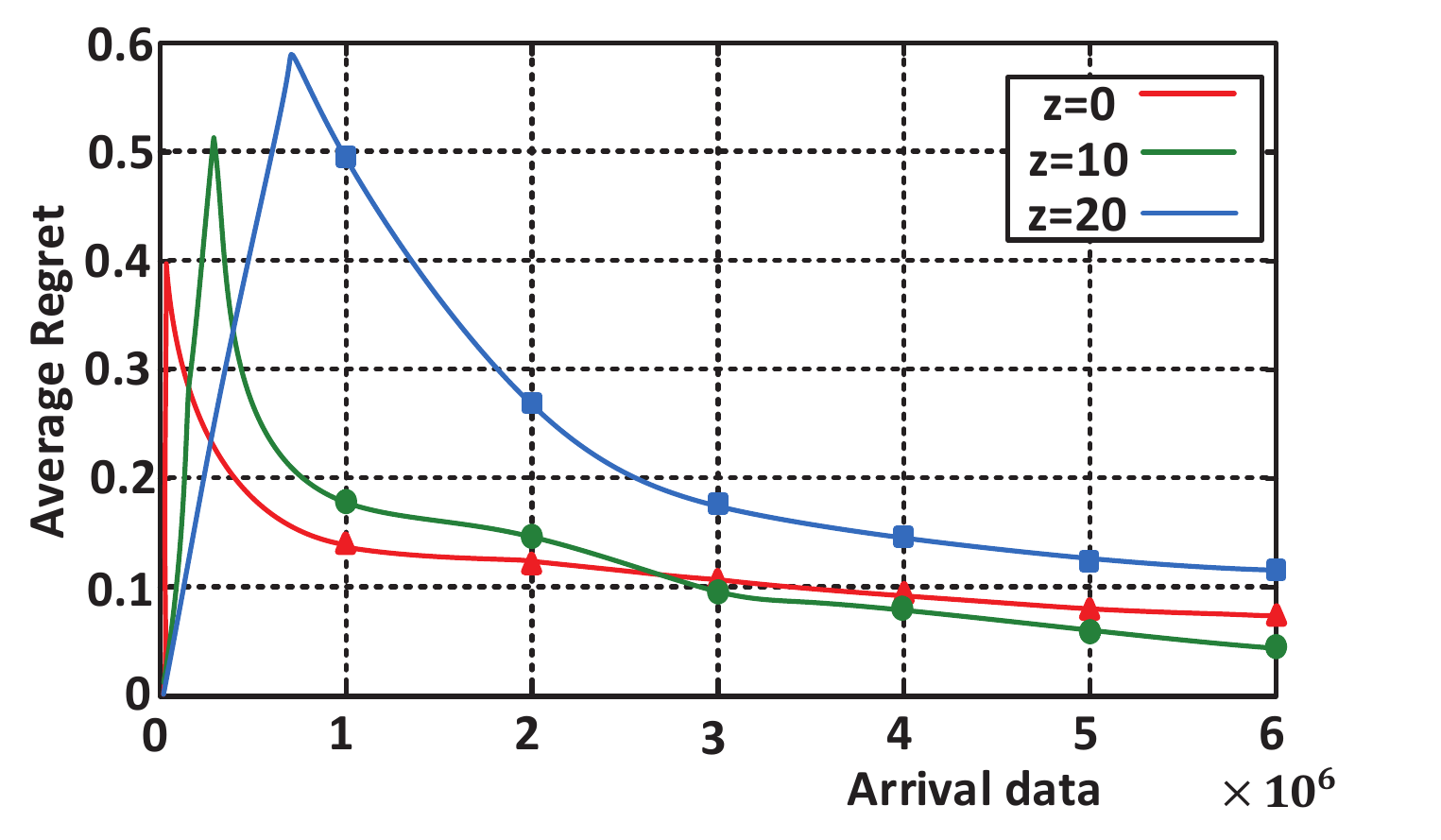}
	\!\caption{\!\!Comparison\!  of\! Average\!  Regret\! with\!  Different\!   $z$}
	\label{fig:digraph7}
\end{figure}


\begin{table}
	\caption{Average Accuracies of DSRHT}
	\begin{center}
		\begin{tabular}{c|cccc}
			\toprule
			\diagbox{Number$\times 10^6 $}{DS factor $z$} & $z=0$ & $z=10$ & $z=20$  \\
			\midrule
			1 & 85.34\% & 82.67\% &  51.10\%  \\
			2 & 87.62\% & 86.98\% &  72.94\%  \\
			3 & 89.92\% & 90.49\% &  81.37\%  \\
			4 & 90.45\% & 91.50\% &  85.79\%  \\
			5 & 91.09\% & 92.03\% &  88.33\%  \\
			6 & 91.87\% & 92.89\% &  89.04\%  \\
			\bottomrule
		\end{tabular}
	\end{center}
\end{table}

In Fig. 8 and Fig. 9 we compare the RHT algorithm with ACR and HCT.
From the Fig. 8 (Regret diagram), we can get that our method is better than the two others which has less regret from the beginning.
The HCT algorithm performs better than ACR when it starts.
With time going on, the ACR's regret comes to be lower than HCT.
From the Fig. 9 (Average Regret diagram), HCT's average regret is less than that of ACR at first, the results also showing that ACR performs slightly better than HCT finally.

Table \Rmnum{2} records the average accuracies which is the total rewards divided by the number of training data (denoted by ``Num").
We find that when the time increases, all the performance of three algorithms can get promoted.
Our algorithm has the highest accuracies during the learning period.
The ACR performs not good when the process starts, whose accuracy is 65.43\% and is worse than that of HCT.
Finally, ACR converges to 88.79\% but HCT is still 83.98\%.
When it comes to our algorithm, it's 91.87\% which is much better than HCT.

Fig. 10 and Fig. 11 analyze the DSRHT algorithm by using different parameters $z$ as 0, 10 and 20.
From the diagrams we find that comparing with $z=0$, $z=10$ is not as well as $z=0$ at the beginning but outperforms it in the long run.
However, when $z=20$,
the algorithm has taken a lot of time to start recommend course precisely.
Even if finally the accuracy of $z=10$ closes to the results that of other two algorithms at the end, the effect is not as well as we expect.

Table \Rmnum{3} illustrates the accuracy more precisely.
When the training number is less than $4 \times 10^6 $, the condition that $z=20$ is the worst in the three conditions.
After that, it come to catch the RHT 91.09\% with 88.33\%.
Thus we can see selecting the distributed storage number cannot pursuit the quantity only, whether it's  makes sense as well in practice.

As for the storage analysis, we use the detailed information of courses to represent courses data,
and the whole course storage is 517.68 TB.
To get more intuition, we use the ratio of actual space occupied and course space occupied to denote storage ratio.
From table \Rmnum{4} we know that ACR\cite{ACR} algorithm is not suitable for real big data since the storage ratio reaches 24.287 TB.
HCT\cite{HCT} algorithm performs well in space complexity which is better than RHT.
As for DSRHT, the storage ratio is 4.118 TB which is less than HCT and nearly half of RHT.

 \begin{table}
 	\caption{Average Storage Cost}	
 	\begin{center}
 		\begin{tabular}{|c|c|c|c|c|} \hline
 			& ACR\cite{ACR} & HCT\cite{HCT} & RHT & \!\!\!\!\! DSRHT \!($z \! = \! 10$) \!\!\!\!\! \\ \hline
 			\!\!\!\!\! Storage Cost (TB) \!\!\!\!\! & 12573 & 2762 & 4123 & 2132 \\ \hline
 			\!\!\!\!\! Storage Ratio \!\!\!\!\! & 24.287 & 5.335 & 7.964 & 4.118 \\ \hline
 		\end{tabular}
 	\end{center}
 \end{table}


\section{Conclusion}
This paper has presented RHT and DSRHT algorithms for the courses recommendation in MOOC big data. Considering the individualization in recommender system, we introduce the context-awareness into our algorithm.
They are suitable for the tremendously huge ad changeable datasets in the future MOOC.
Meanwhile,  they  can achieve  the linear time  and space complexity,  and can achieve the sublinear space complexity in the optimal condition.
Furthermore, we use distributed storage to relieve the storing pressure and make it more suitable for big data. Experiment results
verifies the superior performance of RHT and DSRHT  when comparing with existing related algorithms.

\section*{Appendix A Proof of Lemma 2}
\begin{proof}
	To the first term of in Lemma 1, we take $c_j, \,c_k \in \mathcal{N}_{h,i}^{P_t} $ and  $c_k=c^{P_t*}$ for all context $x_i \in \mathcal{X}$, then we can get that
	\begin{IEEEeqnarray*}{l}
		\!\!\!\!\!\begin{array}{l}
			f({r^{{P_t}}_{c^{P_t*}}}) - f(r_{{c_j}}^{{P_t}}) \le diam(\mathcal{N}_{h,i}^{P_t})+{L_X}{({{\sqrt {{d_X}} } \over {{n_T}}})^\alpha } \\
			\;\,\,\quad\quad\quad\quad\quad\quad\quad\le {k_1}{({m})^h}{{ + }}{L_X}{({{\sqrt {{d_X}} } \over {{n_T}}})^\alpha },
			\IEEEyesnumber \label{eq:Lp2.1}
		\end{array}
	\end{IEEEeqnarray*}
	where $c^{P_t*}$ is the best course whose reward is highest in the context sub-hypercube ${P_t}$.
	We note the event when the path go through the region $\mathcal{N}_{h,i}^{{P_t}}$ as event $\mathcal\{ {N}_{h,i}^{{P_t}} \in \ell _{H,I}^{{P_t*}} \},$
	therefore,
	\begin{IEEEeqnarray*}{l}
		\!\!\!\!\!\begin{array}{l}
			\; \mathbb{P}\left\{ B_{h,i}^{{P_t}}(T^{P_t}) \le f({r^{{P_t}}_{c^*}}){{\;  and \;}}T_{h,i}^{{P_t}}(T^{P_t}) \ge 1\right\}
			\\ \,\,\,\,\,\,=   \mathbb{P}\bigg\{ \hat \mu _{h,i}^{{P_t}}(T^{P_t}) + \sqrt {{k_2}\ln T/{{T}}_{h,i}^{{P_t}}(T^{P_t})}  + {k_1}{({{m}})^h}
			\\\;\;\;\;\;\;\;\;\;\;\;\;\;\;\;\;\;\;\;\;\;\;\;\;\; + {L_X}{({{\sqrt {{d_X}} } \over {{n_T}}})^\alpha }
			\le f({r^{{P_t}}_{c^*}})\;{ and}\;T_{h,i}^{{P_t}}(T^{P_t}) \ge 1\bigg\}
			\\ \,\,\,\,\,\, =  \mathbb{P}\bigg\{ \!\! \left[\hat \mu _{h,i}^{{P_t}}(T^{P_t}\!) \!+\! {k_1}{({{m}})^h} \!+\! {L_X}{({{\sqrt {{d_X}} }  \over {{n_T}}})^\alpha } \!-\! f({r^{{P_t}}_{c^*}})\right]\!{{T}}_{h,i}^{{P_t}}(T^{P_t}\!)
			\\\;\;\;\;\;\;\;\;\;\;\;\;\;\;\;\;\;\;\;\;\;\ \le  - \sqrt {{k_2}(\ln T){{T}}_{h,i}^{{P_t}}(T^{P_t})} \;{{and}}\;T_{h,i}^{{P_t}}(T^{P_t}) \ge 1\bigg\}
			\\
		\end{array}
	\end{IEEEeqnarray*}
	\begin{IEEEeqnarray*}{l}
		\!\!\!\!\!\begin{array}{l}	\,\,\,\,\,\, = \mathbb{P}\bigg\{ \sum\limits_{n = 1}^{T^{P_t}} {\left( r_{c_n}^{{P_t}}(n) - f(r_{{c_n}}^{Pt})\right)\mathbb{I}\left\{ \mathcal{N}_{h,i}^{{P_t}} \in \ell _{H,I}^{{P_t}}\right\} }
			\\
			\;\;\;\;\;\;\;\;\;\;\; +\!\! \sum\limits_{n = 1}^{T^{P_t}} \!\!{\left[\!f(\!r_{{c_n}}^{P_t}) \!+\! {k_1}{{(\!{{m}}\!)}^h} \!\!\!+\!\! {L_X}{{(\!{{\sqrt {{d_X}} } \over {{n_T}}}\!)}^\alpha } \!\!\!\!-\!\! f(\!{r^{{P_t}}_{c^*}}\!)\right]\!}
			{\mathbb{I}\!\left\{\! \mathcal{N}_{h,i}^{{P_t}} \!\in\! \ell _{H,I}^{{P_t}}\! \right\} }
			\\ \;\;\;\;\;\;\;\;\;\,\;\;\;\;\;\;\;\;\;\;\;\;\; \le  - \sqrt {{k_2}(\ln T){{T}}_{h,i}^{{P_t}}(T^{P_t})} \;{{and}}\;T_{h,i}^{{P_t}}(T^{P_t}) \ge 1\bigg\}
			\\ \,\,\,\,\,\, \le \mathbb{P}\bigg\{ \sum\limits_{n = 1}^{T^{P_t}} {( r_{c_n}^{{P_t}}(n) - f(r_{{c_n}}^{Pt}))\mathbb{I}\{ \mathcal{N}_{h,i}^{{P_t}} \in \ell _{H,I}^{{P_t}}\} }
			\\ \;\;\;\;\;\;\;\;\;\;\;\;\;\;\;\;\;\;\;\;\;\;\le  - \sqrt {{k_2}(\ln T){{T}}_{h,i}^{{P_t}}(T^{P_t})} \;{{and}}\;T_{h,i}^{{P_t}}(T^{P_t}) \ge 1\bigg\}.
		\end{array}
	\end{IEEEeqnarray*}	
	
	The last inequation is based on the expression ($\ref{eq:Lp2.1}$), since the second term is positive and we drop it to get the last expression.
	
	For the convenience of illustration, we pick the n when ${\mathbb{I}\left\{ \mathcal{N}_{h,i}^{{P_t}} \in \ell _{H,I}^{{P_t}}\right\} }$ is equal to 1.
	We use
	$\mathord{\buildrel{\lower3pt\hbox{$\scriptscriptstyle\smile$}}\over r} _c^{{P_t}}$
	to indicate the $r_{c_n}^{{P_t}}$ happened in ${\mathbb{I}\left\{ \mathcal{N}_{h,i}^{{P_t}} \in \ell _{H,I}^{{P_t}}\right\} }$.
	Thus,
	\begin{IEEEeqnarray*}{l}
		\!\!\!\!\!\begin{array}{l}
			\mathbb{P}\bigg\{ \sum\limits_{n = 1}^{T^{P_t}} {\left( r_{c_n}^{{P_t}}(n) - f(r_{{c_n}}^{Pt})\right)\mathbb{I}\left\{ \mathcal{N}_{h,i}^{{P_t}} \in \ell _{H,I}^{{P_t}}\right\} }
			\\ \;\;\;\;\;\;\;\;\;\;\;\;\;\;\;\;\;\;\;\;\;\;\; \le  - \sqrt {{k_2}(\ln T){{T}}_{h,i}^{{P_t}}(T^{P_t})} \;{{and}}\;T_{h,i}^{{P_t}}(T^{P_t}) \ge 1\bigg\}
			\\ \,\,\,\,\,\,  \le  \mathbb{P}\bigg\{ \sum\limits_{n = 1}^{T^{P_t}} {\left( r_{c_n}^{{P_t}}(n) - f(r_{{c_n}}^{Pt})\right)\mathbb{I}\left\{ \mathcal{N}_{h,i}^{{P_t}} \in \ell _{H,I}^{{P_t}}\right\} }
			\\ \;\;\;\;\;\;\;\;\;\;\;\;\;\;\;\;\;\;\;\;\;\;\; \le  - \sqrt {{k_2}(\ln T){{T}}_{h,i}^{{P_t}}(T^{P_t})}
			\;{{and}}\;T_{h,i}^{{P_t}}(T^{P_t}) \ge 1\bigg\}
			\\ \,\,\,\,\,\, = \mathbb{P}\Bigg\{ \sum\limits_{n = 1}^{{{T}}_{h,i}^{{P_t}}(T^{P_t})} {\left(\mathord{\buildrel{\lower3pt\hbox{$\scriptscriptstyle\smile$}}\over
					r} _c^{{P_t}} - \mathord{\buildrel{\lower3pt\hbox{$\scriptscriptstyle\smile$}}\over
					r} _{{c_n}}^{Pt}\right)}
			\\ \;\;\;\;\;\;\;\;\;\;\;\;\;\;\;\;\;\;\;\;\;\;\; \le  - \sqrt {{k_2}(\ln T){{T}}_{h,i}^{{P_t}}(T^{P_t})}
			{{\;and}}\;T_{h,i}^{{P_t}}(T^{P_t}) \ge 1\Bigg\} 	
			\\ \,\,\,\,\,\, \le \sum\limits_{n = 1}^{T^{P_t}} {\mathbb{P}\bigg\{ \sum\limits_{j = 1}^n {\left(f(\mathord{\buildrel{\lower3pt\hbox{$\scriptscriptstyle\smile$}}\over
						r} _c^{{P_t}}) - f(\mathord{\buildrel{\lower3pt\hbox{$\scriptscriptstyle\smile$}}\over
						r} _{{c_j}}^{Pt})\right)}  \le  - \sqrt {{k_2}(\ln T)n} \bigg\} }.
		\end{array}
	\end{IEEEeqnarray*}
	
	We consider the situation when $n = 1,2...T_{h,i}^{{P_t}}(T^{P_t})$ and the fact that ${{T}}_{h,i}^{{P_t}}(T^{P_t}) \le T^{P_t}$.
	Besides, the last inequation use the union bound theory and loose the threshold
	\begin{IEEEeqnarray*}{l}
		\!\!\!\!\!\begin{array}{l}
			\sum\limits_{n = 1}^{T^{P_t}} \mathbb{P}\bigg\{ \sum\limits_{j = 1}^n{\left( f(\mathord{\buildrel{\lower3pt\hbox{$\scriptscriptstyle\smile$}}\over
					r} _c^{{P_t}}) - f(\mathord{\buildrel{\lower3pt\hbox{$\scriptscriptstyle\smile$}}\over
					r} _{{c_j}}^{Pt})\right)  \le  - \sqrt {{k_2}(\ln T)n} \bigg\} }
			\\ \,\,\,\,\,\, \le \sum\limits_{n = 1}^{T^{P_t}} {\exp ( - 2{k_2}\ln T)}  \le {(T^{P_t})^{ - 2{k_2} + 1}}.
			\IEEEyesnumber \label{eq:Lp2.4}
		\end{array}
	\end{IEEEeqnarray*}
	Note that the sum of time T represents the contextual sum of time since the number of courses in the context sub-hypercube is stochastic.
	And for the convenience, we use T as the sum of time.
	With the help of Hoeffding-Azuma inequality\cite{HAine}, we get the conclusion.

	With the help of the assumption of range over $q$,
	we can get
	\begin{IEEEeqnarray*}{l}
		\!\!\!\!\!\begin{array}{l}
			{{D_{C(h,i)}^{{P_t}} - {k_1}{{({m})}^h} - {L_X}{{({{\sqrt {{d_X}} } \over {{n_T}}})}^\alpha }} \over 2} \ge \sqrt {{{{{{k}}_2}{{\ln T}}} \over q}}.
			\IEEEyesnumber \label{eq:Lp3}
		\end{array}
	\end{IEEEeqnarray*}
	Thus, the
	\begin{IEEEeqnarray*}{l}
		\!\!\!\!\!\begin{array}{l}
			\mathbb{P}\left\{ B_{h,i}^{{P_t}}(T^{P_t}) > f({r^{{P_t}}_{c^{P_t*}}})\;{{and}}\;T_{h,i}^{{P_t}}(T^{P_t}) \ge q\right\} 	
			\\ \,\,\,\,\,\, = \mathbb{P}\Big\{ \hat \mu _{h,i}^{{P_t}}(T^{\!P_t}\!) \!+\!\! \sqrt {{k_2}\! \ln T/{{T}}_{h,i}^{{P_t}}(T^{\!P_t}\!)}  \!+\! {k_1}{({{{m}}})^h}\!
			\!+\! {L_X}{(\!{{\sqrt {{d_X}} } \over {{n_T}}}\!)^\alpha }
			\\
			\;\;\;\;\;\;\;\;\;\;\;\;\;\;\;\;\;\;\;\;\;\;\; > f(r^{P_t}_{c^{P_t*}(h,i)}) + D_{C(h,i)}^{{P_t}}{{ }}\;{ and}\;T_{h,i}^{{P_t}}(T^{P_t}) \ge q\Big\}
			\\ 	\,\,\,\,\,\, \le \mathbb{P}\Big\{ \hat \mu _{h,i}^{{P_t}}(T^{P_t}) + \sqrt {{{{k_2}{\mathop{ \ln T}\nolimits} } \over q}}  + {k_1}{({{m}})^h} + {L_X}{({{\sqrt {{d_X}} } \over {{n_T}}})^\alpha }
			\\\;\;\;\;\;\;\;\;\;\;\;\;\;\;\;\;\;\;\;\;\;\;\; > f(r^{P_t}_{c^{P_t*}(h,i)}) + D_{C(h,i)}^{{P_t}}\;{ and}\;T_{h,i}^{{P_t}}(T^{P_t}) \ge q\Big\}
			\\ \,\,\,\,\,\, = \mathbb{P}\bigg\{ \![\hat \mu _{h,i}^{{P_t}}(T^{\!P_t\!}) \!-\! f(r^{P_t}_{\!c^{P_t\!*\!}(h,i)\!})] \!\!>\!\! [\!{{D_{C(h,i)}^{{P_t}} \!-\! {k_1}{{({{m}})}^h} \!-\! {L_X}{{({{\sqrt {{d_X}} } \over {{n_T}}})}^\alpha }} \over 2}\!]
			\\\;\;\;\;\;\;\;\;\;\;\;\;\;\;\;\;\;\;\;\;\;\;\;\;\;\;\;\;\;\;\;\;\;\;\;\;\;\;\;\;\;\;\;\;\;\;\;\;\;\;\;\;\;\;\;\;\;\;\;\; { and}\;T_{h,i}^{{P_t}}(T^{P_t}) \ge q\bigg\}.
		\end{array}
	\end{IEEEeqnarray*}
	
	When we multiply $T_{h,i}^{{P_t}}(T^{P_t})$ with both sides, we can get the inequations below.
	
	\begin{IEEEeqnarray*}{l}
		\!\!\!\!\!\begin{array}{l}
			\mathbb{P}\bigg\{ \![\hat \mu _{h,i}^{{P_t}}(T^{\!P_t\!}) \!-\! f(r^{P_t}_{\!c^{P_t\!*\!}(h,i)\!})] \!\!>\!\! [\!{{D_{C(h,i)}^{{P_t}} \!-\! {k_1}{{({{m}})}^h} \!-\! {L_X}{{({{\sqrt {{d_X}} } \over {{n_T}}})}^\alpha }} \over 2}\!]
			\\\;\;\;\;\;\;\;\;\;\;\;\;\;\;\;\;\;\;\;\;\;\;\;\;\;\;\;\;\;\;\;\;\;\;\;\;\;\;\;\;\;\;\;\;\;\;\;\;\;\;\;\;\;\;\;\;\;\;\;\; { and}\;T_{h,i}^{{P_t}}(T^{P_t}) \ge q\bigg\}
		\end{array}
	\end{IEEEeqnarray*}
	\begin{IEEEeqnarray*}{l}
		\!\!\!\!\!\begin{array}{l}
			\\  \,\,\,\,\,\, = \mathbb{P}\Big\{ \sum\limits_{n = 1}^{T^{P_t}} {( r_n^{{P_t}}(n) - f(r_{h,i}^{Pt}))\mathbb{I}\{ \mathcal{N}_{h,i}^{{P_t}} \in \ell _{H,I}^{{P_t}}\} }
			\\ \;\;\quad\quad \!>\!\! [\!{{D_{\!C(h,i)}^{{P_t}} \!-\! {k_1}{{({{{m}}})}^h} \!-\! {L_{\!X}}{{\!({{\sqrt {{d_X}} } \over {{n_T}}})}^\alpha }} \over 2}\!] T_{h,i}^{{P_t}}(T^{\!P_t\!})\;
			{{and}}\;T_{h,i}^{{P_t}}(T^{\!P_t\!}) \!\!\ge\!\! q\Big\}.
		\end{array}
	\end{IEEEeqnarray*}
	
	With the union bound and the Hoeffding-Azuma inequality\cite{HAine}, we can get that
	\begin{IEEEeqnarray*}{l}
		\!\!\!\!\!\begin{array}{l}
			\mathbb{P}\Big\{ \sum\limits_{n = 1}^{T^{P_t}} {\left( r_n^{{P_t}}(n) - f(r_{{c_n}}^{Pt})\right)\mathbb{I}\left\{ \mathcal{N}_{h,i}^{{P_t}} \in \ell _{H,I}^{{P_t}}\right\} }
			\\ \quad\quad \!>\!\! [\!{{D_{\!C(h,i)}^{{P_t}} \!-\! {k_1}{{({{{m}}})}^h} \!-\! {L_{\!X}}{{\!({{\sqrt {{d_X}} } \over {{n_T}}})}^\alpha }} \over 2}\!] T_{h,i}^{{P_t}}(T^{\!P_t\!})\;
			{{and}}\;T_{h,i}^{{P_t}}(T^{\!P_t\!}) \!\!\ge\!\! q\Big\}
			\\  \,\,\,\,\,\, \le {({{T^{P_t}}})^{ - 2{k_2} + 1}}.
		\end{array}
	\end{IEEEeqnarray*}
	
	According to Lemma 1 and the prerequisite in Lemma 2, we select upper bound of $q$ as
	${{4{k_2}{{\ln T}}} \over {{{\left[D_{C(h,i)}^{{P_t}} \!-\! {k_1}{{({m})}^h} \!-\! {L_X}{{({{\sqrt {{d_X}} } \over {{n_T}}})}^\alpha }\right]}^2}}} \!+\! 1$.
	Thus,
	\begin{IEEEeqnarray*}{l}
		\!\!\!\!\!\begin{array}{l}
			\mathbb{E}\!\left[T_{h,i}^{{P_t}}(T^{P_t})\right]
			\!\le\!\! \sum\limits_{n = q + 1}^{T^{P_t}} {\mathbb{P}\bigg\{\!\! \left[B_{h,i}^{{P_t}}(n) \!>\! f({r^{{P_t}}_{c^{P_t*}}}){{\; and}}\;T_{h,i}^{{P_t}}(n) \!>\! q\right]}
			\\\;\;\;\;\;\;\;\;\;\;\;\;\;\;\;\;\;\;\;\;\;\;\; {{ \;or}}\ \;\!\!\!\! \left[\!B_{j,{i_{h'\!\!}}}^{{P_t}}(n) \!\!\le\! f(\!{r^{{P_t}}_{c^{P_t*}}}\!){{\;f\!or}}\;{{j}} \!\in\! \{ q \!\!+\!\! 1,...,n \!\!-\!\! 1\} \right]\!\!\bigg\}
			\\\;\;\;\;\;\;\;\;\;\;\;\;\;\;\;\;\;\;\;\;\;\;\;\;\;\;\;\;\;\;\;\;\;\;\;\;\;\;\; + {{4{k_2}{{\ln T}}} \over {{{\left[D_{C(h,i)}^{{P_t}} - {k_1}{{({m})}^h} - {L_X}{{({{\sqrt {{d_X}} } \over {{n_T}}})}^\alpha }\right]}^2}}} +1
			\\ \;\;\;\;\;\;\;\;\;\;\;\;\;\;\;\;\;\;\;\;\;  \le {{4{k_2}{{\ln T}}} \over {{{\left[D_{C(h,i)}^{{P_t}} - {k_1}{{({m})}^h} - {L_X}{{({{\sqrt {{d_X}} } \over {{n_T}}})}^\alpha }\right]}^2}}}
			\!+\!  1
			\\\;\;\;\;\;\;\;\;\;\;\;\;\;\;\;\;\;\;\;\;\;\;\;\;\;\;\;\;\;\;\;\;\;\;\;\;\;\;\;\;\;\; +\!\sum\limits_{n = q + 1}^{T^{P_t}} {\left[{{(T^{P_t})}^{ - 2{k_2}{{ + 1}}}}{{ \!+\! }}{{{n}}^{ - 2{k_2} + 2}}\right]}.
		\end{array}
	\end{IEEEeqnarray*}
	And we take the constant ${k_2} \ge 1$,
	\begin{IEEEeqnarray*}{l}
		\!\!\!\!\!\begin{array}{l}
			1 + \sum\limits_{n = q + 1}^{T^{P_t}} {\left[{{(T^{P_t})}^{ - 2{k_2}{{ + 1}}}}{{ + }}{{{n}}^{ - 2{k_2} + 2}}\right]}  \le 4 \le M,
			\IEEEyesnumber \label{eq:Lp3.5}
		\end{array}
	\end{IEEEeqnarray*}
	
	thus we can get the conclusion Lemma 2.	
\end{proof}

\section*{Appendix B Proof of Theorem 2}
\begin{proof}
	Based on the segmentation, the regret can be presented with
	\begin{IEEEeqnarray*}{l}
		\!\!\!\!\!\begin{array}{l}
			\mathbb{E}[R({T})] = \mathbb{E}[{R_1}({T})] + \mathbb{E}[{R_2}({T})] + \mathbb{E}[{R_3}({T})] + \mathbb{E}[{R_4}({T})].
		\end{array}
	\end{IEEEeqnarray*}
	For $\mathbb{E}[{R_1}({T})]$, since it's the same as the Algorithm 1, so we can get the first term as
	\begin{IEEEeqnarray*}{l}
		\!\!\!\!\!\begin{array}{l}
			\mathbb{E}[R_1 (T)] \le  {4\left[k_1 ({m } )^H  + L_X ({{\sqrt {d_X } } \over {n_T }})^\alpha  \right]T}  .
			\IEEEyesnumber \label{eq:T2.1}
		\end{array}
	\end{IEEEeqnarray*}
	
	The depth is from $z$ to $H$, revealing that $H>z$.
	To satisfy this, we suppose
	${2^H} \ge {\left({T \over {\ln T}}\right)^{{{{d_X} + \alpha{d_C} } \over {{d_X} + \alpha({d_C} + 2)}}}}.$
	Since the exploration process started from depth $z$, the depth we can select satisfy the inequation above.
	Thus the second term's regret bound is
	\begin{IEEEeqnarray*}{l}
		\!\!\!\!\!\begin{array}{l}
			\mathbb{E}[{R_2}({T})]
			\le \sum\limits_{{P_t}} {\sum\limits_{h = z}^H {4\left[{k_1}{{({m})}^h} + {L_X}{{({{\sqrt {{d_X}} } \over {{n_T}}})}^\alpha }\right]} \left|\phi _h^{{P_t}}\right|}
			\\\;\;\;\;\;\;\;\;\;\;\;\;\;\;\; \le {{4K{{({n_T})}^{{d_X}}}} \over {{{\left[{k_1}{{({m})}^h}\right]}^{{d_C}}}}}\sum\limits_{h = z}^H {4\left[{k_1}{{({m})}^h} + {L_X}{{({{\sqrt {{d_X}} } \over {{n_T}}})}^\alpha }\right]}.
			\IEEEyesnumber \label{eq:T2.2}
		\end{array}
	\end{IEEEeqnarray*}
	
	We choose the context sub-hypercube whose regret bound is biggest to continue the inequation ($\ref{eq:T2.2}$).
	And as for the third term, the regret bound is
	\begin{IEEEeqnarray*}{l}
		\!\!\!\!\!\begin{array}{l}
			\;\mathbb{E}[{R_3}({T})]
			\!\le\! \sum\limits_{{P_t}} \!{\sum\limits_{{{h}} = z}^H \!{4\! \left[{k_1}{{({m})}^{h-1}} \!\!\!+\!\! {L_X}{{({{\sqrt {{d_X}} } \over {{n_T}}})}^\alpha }\!\right]} \!\! \sum\limits_{\mathcal{N}_{h,i}^{{P_t}} \in \Gamma^{P_t}_3} \!{\left|{{(\phi _h^{{P_t}})}^c} \right|} } .
		\end{array}
	\end{IEEEeqnarray*}
	
	We notice that since the regions in $\Gamma _3^{{P_t}}$ is the child region of $\Gamma _2^{{P_t}}$.
	To be more specific, in the binary tree, the child regions is more than parent regions but less than twice, thus the number of top regions in $\Gamma _3^{{P_t}}$ is less than twice of $\Gamma _2^{{P_t}}$.
	\begin{IEEEeqnarray*}{l}
		\!\!\!\!\!\begin{array}{l}
			\; \sum\limits_{{P_t}} {\sum\limits_{{{h}} = z}^H {4\left[{k_1}{{({m})}^{h-1}} + {L_X}{{({{\sqrt {{d_X}} } \over {{n_T}}})}^\alpha }\right]} \sum\limits_{\mathcal{N}_{h,i}^{{P_t}} \in {\Gamma^{P_t}_3}} {\left|{{(\phi _h^{{P_t}})}^c}\right|} } \;\;
			\\ \;\,\,\,\,\,\, \le \sum\limits_h \Bigg\{ {{32k_2 K(n_{T} )^{d_{X} } \ln T} \over {\left[ {k_1 (m)^h } \right]^{d_C  + 1} \left[ {k_1 (m)^h  + L_{X} ({{\sqrt {d_X } } \over {n_T }})^{\alpha } } \right]}} \\
			\quad \quad \quad \quad \quad \quad \quad \quad \quad \quad \; + {{8MK(n_{T} )^{d_{X} } \left[ {k_1 (m)^h  + L_{X} ({{\sqrt {d_X } } \over {n_T }})^{\alpha } } \right]} \over {m\left[ {k_1 (m)^h } \right]^{d_C } }} \Bigg\}. 	
		\end{array}
	\end{IEEEeqnarray*}
	From the upper bounds of regret $\mathbb{E}[{R_1}({T})]$, $\mathbb{E}[{R_2}({T})]$, $\mathbb{E}[{R_3}({T})]$, we can get that the three upper bound is the same as algorithm RHT.
\end{proof}

\end{document}